\documentclass{article}

\usepackage{microtype}
\usepackage{graphicx}
\usepackage{subcaption}
\usepackage{booktabs} 

\usepackage{hyperref}

\usepackage[arxiv]{icml2024}

\usepackage{amsmath}
\usepackage{amssymb}
\usepackage{mathtools}
\usepackage{amsthm}

\usepackage[capitalize,noabbrev]{cleveref}

\theoremstyle{plain}
\newtheorem{theorem}{Theorem}[section]
\newtheorem{proposition}[theorem]{Proposition}
\newtheorem{lemma}[theorem]{Lemma}

\theoremstyle{definition}

\newtheorem{assumption}[theorem]{Assumption}
\theoremstyle{remark}
\newtheorem{remark}[theorem]{Remark}

\newcommand{\R}{\mathbb{R}}
\newcommand{\TV}{\mathrm{TV}}
\newcommand{\KL}{\mathrm{KL}}
\newcommand{\E}{\mathbb{E}}
\newcommand{\pf}{p}
\newcommand{\pfN}{{\mathsf{p}}}
\newcommand{\ptarget}{p_\ast}
\newcommand{\ptargetN}{{\mathsf{p}_\ast}}
\newcommand{\ptargetNgamma}{\mathsf{p}_{\ast}^\gamma}
\newcommand{\pb}{q}
\newcommand{\pbN}{{\mathsf{q}}}
\newcommand{\score}{u}

\newcommand{\pbs}{\bar{\mathsf{q}}}
\newcommand{\pbsN}{\widehat{\mathsf{q}}}
\newcommand{\pcond}{p}
\newcommand{\scoreN}{s^N}
\newcommand{\scoreNy}{s^N_{\{y_i\}}}
\newcommand{\errorNy}{E_{\{y_i\}}}
\newcommand{\errorNyt}{E^t_{\{y_i\}}}
\newcommand{\epsscore}{\varepsilon_{\text{score}}}
\makeatletter
\def\url@leostyle{%
	\@ifundefined{selectfont}{\def\UrlFont{\sf}}{\def\UrlFont{\small\ttfamily}}}
\makeatother
\definecolor{darkgreen}{rgb}{0,0.4,0}

\usepackage[textsize=tiny]{todonotes}

\icmltitlerunning{}

\begin{document}

\twocolumn[
\icmltitle{A Good Score Does not Lead to A Good Generative Model}



\icmlsetsymbol{equal}{*}

\begin{icmlauthorlist}
\icmlauthor{Sixu Li}{affiliation 1}
\icmlauthor{Shi Chen}{affiliation 2}
\icmlauthor{Qin Li}{affiliation 2}
\end{icmlauthorlist}

\icmlaffiliation{affiliation 1}{Department of Statistics, University of Wisconsin-Madison, Madison WI, USA}
\icmlaffiliation{affiliation 2}{Department of Mathematics, University of Wisconsin-Madison, Madison WI, USA}

\icmlcorrespondingauthor{Qin Li}{qinli@math.wisc.edu}

\icmlkeywords{Machine Learning, ICML}

\vskip 0.3in
]



\printAffiliationsAndNotice{} 

\begin{abstract}
Score-based Generative Models (SGMs) is one leading method in generative modeling, renowned for their ability to generate high-quality samples from complex, high-dimensional data distributions. The method enjoys empirical success and is supported by rigorous theoretical convergence properties. In particular, it has been shown that SGMs can generate samples from a distribution that is close to the ground-truth if the underlying score function is learned well, suggesting the success of SGM as a generative model. We provide a counter-example in this paper. Through the sample complexity argument, we provide one specific setting where the score function is learned well. Yet, SGMs in this setting can only output samples that are Gaussian blurring of training data points, mimicking the effects of kernel density estimation. The finding resonates a series of recent finding that reveal that SGMs can demonstrate strong memorization effect and fail to generate.


\end{abstract}

\section{Introduction}
\label{sec: intro}
Generative modeling aims to understand the dataset structure so to generate similar examples. It has been widely used in image and text generation \cite{wang2018high,huang2018introvae,rombach2022high,li2022diffusion,gong2022diffuseq}, speech and audio synthesis \cite{donahue2018adversarial,kong2020hifi, kong2020diffwave,huang2022fastdiff}, and even the discovery of protein structures \cite{watson2023novo}.

Among the various types of generative models, Score-based Generative Models (SGMs) \cite{song2020score,ho2020denoising,karras2022elucidating} have recently emerged as a forefront method, and achieved state-of-the-art empirical results across diverse domains. It views the data structure of existing examples coded in a probability distribution, that we call the target distribution. Once SGM learns the target distribution from the data, it generates new samples from it.


Despite their empirical successes, a thorough theoretical understanding of why SGMs perform well remains elusive. A more fundamental question is:

\emph{What are the criteria to evaluate the performance of a generative model}?

Heuristically, two key components of generative models are ``imitating" and ``generating". The ``imitating'' is about learning from the existences, while generating calls for creativity to produce new. A successful generative model should exhibit both \textit{imitation ability}, so to produce samples that resemble the training data, and at the same time, manifest \textit{creativity}, and generate samples that are not mere replicas of existing ones.

In the past few years, significance theoretical progresses have been made on assessing the \textit{imitation ability} of SGMs. In particular, recently made available theory provides a very nice collection of error bounds to evaluate the difference between the learned distribution and the ground-truth distribution. Such discussion has been made available in various statistical distances, including total variation, KL divergence, Wasserstein distance and others. These discoveries suggest that SGMs have strong imitation ability, i.e. can approximate the ground-truth distribution well if the score function (gradient of log-density) of the target distribution along the diffusion process can be effectively learned. 

We would like to discuss the other side of the story: Relying solely on these upper error bounds might be misleading in assessing the overall performance of SGMs. In particular, this criterion does not adequately address the issue of memorization -- the possibility that the produced samples are simply replicas of the training data. In other words, SGMs with strong imitation ability can be lack of creativity.

\subsection{A toy model argument}
At the heart of our argument is that a simple Kernel Density Estimation (KDE) of the target ground-truth distribution can be arbitrarily close. Yet, drawing a sample from the ground-truth and drawing one from a KDE presents very different features. The latter fails on the task of ``generation."

To be mathematically more precise, let $\ptarget(x)$ be the ground-truth distribution, and $\{y_i\}_{i=1}^N$ be a set of i.i.d samples drawn from it. The empirical distribution is $\ptargetN := \frac{1}{N} \sum_{i=1}^N \delta_{y_i}$. 
We denote $\ptargetNgamma := \ptargetN * \mathcal{N}_{\gamma}$ the distribution obtained by smoothing $\ptargetN$ with a Gaussian kernel $\mathcal{N}_{\gamma} := \mathcal{N}(0, \gamma^2 I_{d\times d})$. 
Such definition naturally puts $\ptargetNgamma$ as one kind of Kernel Density Estimation (KDE) of $\ptarget$ with the bandwidth $\gamma$.

It is intuitive that when the sample size $N$ is large, and the bandwidth $\gamma$ is properly chosen, the KDE $\ptargetNgamma$ approximates the true distribution $q$. In the most extreme case, when the bandwidth $\gamma \to 0$, the kernel density estimate $\ptargetNgamma$ degenerates to the empirical distribution $\ptargetN$. Throughout the paper we view the empirical distribution as a special case of KDE.

Though $\ptarget$ and $\ptargetNgamma$ are close, generating samples from $\ptarget$ and from $\ptargetNgamma$ are drastically different stories. 
Drawing from $\ptarget$ amounts to generate a completely new sample, independent of the dataset, while generating from $\ptargetNgamma$ essentially means selecting a sample uniformly from the set $\{y_i\}_{i=1}^N$ and then applying a Gaussian blurring. 
Regardless of how close $\ptargetNgamma$ approximates the ground-truth $\ptarget$, sampling from KDE ultimately gives replicas of the existing samples.

Would SGM be different from KDE? SGM is built on a complicated procedure, incorporating forward noise injection, score matching, and backward sampling processes. The machinery is significantly more convoluted than the straightforward KDE approach. Would it be able to generate new samples?

We are to show in this paper that the perfect SGM is actually a KDE itself. The mathematical statement is presented in Theorem \ref{thm: DDPM resemble KDE}. The ``perfect'' means the minimizer of the empirical score matching objective is achieved during the score-matching procedure. We term the learned score function the \textit{empirical optimal score function}. Since SGM equiped with the empirical optimal score function is effectively a KDE, it sees the limitation of KDE and fails to ``generate.'' This phenomenon is clearly demonstrated in Figure~\ref{fig: image generations from kde} with the test conducted over the CIFAR10 dataset.

\begin{figure}[!htb]
\centering
\includegraphics[width=0.48\textwidth]{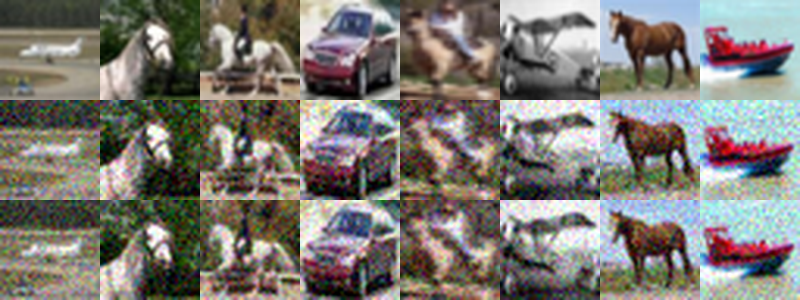}
\caption{Images generated based on CIFAR10 dataset. The first row shows the original images, the second row presents the images blurred according to the Gaussian KDE, and the third row shows images generated by SGM equipped with the perfect score function learned from samples. Both KDE and SGM present simple replica (with Gaussian blurring) of the original images.}
\label{fig: image generations from kde}
\end{figure}



It is important to note that this observation does not contradict existing theories that suggest SGMs can approximate the target distribution $q$ when the score function is accurately learned. Indeed, in Theorem \ref{thm: approximation error of empirical optimal score function} we provide a sample complexity estimate and derive a lower bound of the sample size $N$. When the sample size is sufficiently large, the empirical optimal score function approximates the ground-truth score function. Consequently, according to the existing theories, the output of SGM is a distribution close to the ground-truth target distribution. Yet, two distribution being close is not sufficient for the task of generation.

\subsection{Contributions}
The primary contribution of this paper is presenting a counter-example of score-based generative models (SGMs) with accurate approximated score function, yet producing unoriginal, replicated samples.
Our findings are substantiated through the following two steps:
\begin{itemize}
\item We establish in Theorem \ref{thm: approximation error of empirical optimal score function} the score-matching error of the empirical optimal score function, and present an explicit non-asymptotic error bound with the sample complexity. 
    This result illustrates that the empirical optimal score function satisfies the standard $L^2$ bound on the score estimation error used in the convergence analysis in the existing literature \cite{chen2022sampling,chen2023probability,chen2023restoration,benton2023linear}, which presumably should lead to the conclusion that SGMs equipped with the empirical optimal score function produces a distribution close to the target distribution.
    \item We show in Theorem \ref{thm: DDPM resemble KDE} that SGMs equipped with empirical optimal score function resembles a Gaussian KDE, and thus presents strong memorization effects and fails to produce novel samples.

\end{itemize}

These results combined rigorously demonstrates that the SGM with precise empirical score matching is capable to produce a distribution close to the target, but the procedure does not ensure the efficacy of an SGM in its ability to generate innovative and diverse samples. This observation underscores the limitation of current upper bound type guarantees and highlights the need for new theoretical criteria to assess the performance of generative models.

\textbf{Notations:} Let $\R^d$ to be the $d$-dimensional Euclidean space and $T > 0$ is the time horizon.
Denote $x = (x_1, x_2, \dots, x_d)^\top \in \R^d$ and $t \in [0,T]$ to be the spatial variable and time variable respectively.
We denote $\ptarget$ as the target data distribution supported on a subset of $\R^d$, and indicate the empirical distribution by $\ptargetN$.
The Gaussian kernel with bandwidth $\gamma$ is denoted by $\mathcal{N}_{\gamma} := \mathcal{N}(0, \gamma^2 I_{d\times d})$.
For the special case $\gamma = 1$, i.e. standard Gaussian, we use notation $\pi^d := \mathcal{N}(0, I_{d\times d})$.
We denote the Gaussian KDE with bandwidth $\gamma$ as $\ptargetNgamma := \ptargetN * \mathcal{N}_{\gamma}$.
In general, we use $p_t$ and $q_t$ (or $\pfN_t$ and $\pbN_t$) to represent the laws of forward and backward SDEs at time $t$ respectively (a thorough summary of PDEs and SDEs' notations used in this paper is provided in Appendix \ref{appendix: notations}).
We denote $\delta \in [0, T)$ to be the early stopping time for running SDEs.

\subsection{Literature review}
We are mainly concerned of three distinct lines of research related to SGM performance, as summarized below.

\textbf{Convergence of SGMs.}
The first line of research concerns theoretical convergence properties of SGMs. This addresses the most fundamental performance of the algorithm: What elements are needed for SGM to perform well? In this context, a good performance amounts to generating a new sample from the learned distribution that is close to the ground-truth. This line of research has garnered a large amount of interests, drawing its relation to sampling. For most studies, the analysis becomes quantifying the deviation between distributions generated by SGMs and the ground-truth distributions. This includes the earlier studies such as~\cite{lee2022convergence,wibisono2022convergence,de2021diffusion,de2022convergence,kwon2022score,block2022generative}, and later~\cite{chen2022sampling,chen2023improved,chen2023score,benton2023linear,li2023towards} that significantly relaxed the Lipschitz condition of the score function and achieved polynomial convergence rate. In these discoveries, Girsanov's theorem turns out to be a crucial proof strategy. Parallel to these findings, convergence properties of ODE-based SGMs have also been explored \cite{chen2023probability,chen2023restoration,benton2023error,albergo2023stochastic,li2023towards}, and comparison to SDE-based SGMs have been drawn.


\textbf{Sample complexity studies of SGMs.}
Another line of research focuses on sample complexity. How many samples/training data points are needed to learn the score? In line with convergence rate analysis, the sample complexity study has been conducted with the criteria set to be $L^2$-approximation of the score function~\cite{block2022generative,cui2023analysis,chen2023score,oko2023diffusion}. The involved techniques range from deploying Rademarcher complexity for certain hypothesis classes, to utilizing specific neural network structures. Often in times, there are also assumptions made on the structure of data.


\textbf{Memorization effect of SGMs.}
The third line of research on SGM concerns its memorizing effect. This line of research was triggered by some experimental discovery and was confirmed by some high profile lawsuits~\cite{newyorktimes}. Experimentally it was found that SGMs, when trained well, tend to produce replicas of training samples~\cite{somepalli2022diffusion,somepalli2023understanding,carlini2023extracting}. This phenomenon draws serious privacy concerns, and motivates studies on the fundamental nature of SGMs: Are SGMs memorizers or generalizers? In~\cite{yoon2023diffusion}, the authors presented a dichotomy, showing through numerical experiments that SGMs can generate novel samples when they fail to memorize training data. Furthermore, when confined to a basis of harmonic functions adapted to the geometry of image features,~\cite{kadkhodaie2023generalization} suggest that neural network denoisers in SGMs might have an inductive bias, aiming the generation. In~\cite{gu2023memorization,yi2023generalization}, the authors derive the optimal solution to the empirical score-matching problem and show that the SGMs equipped with this score function exhibit a strong memorization effect. This suggests that with limited amount of training data and a large neural network capacity, SGMs tend to memorize rather than generalize.

To summarize: the convergence results of SGMs suggest a well-learned score function can be called to produce a sample drawn from a distribution close to the ground-truth, and the studies on the memorization effect of SGMs suggest the new drawings are simple replicas of the training dataset. It is worth noting that the two sets of results do not contradict. In particular, the convergence results do not rule out the explicit dependence of new generated samples on the training data. The connection between the two aspects of SGM performance is yet to be developed, and this is our main task of the current paper. We show that SGMs, despite having favorable convergence properties, can still resort to memorization, in the form of kernel density estimation. The finding underscores the need for a new theoretical framework to evaluate SGMs' performance, taking into account both imitation ability and creativity of SGMs.

\section{Score-based Generative Models}
We provide a brief expository to the Score-based Generative models (SGM)~\cite{song2020score} in this section. Mathematically, SGM is equivalent to denoising diffusion probabilistic modeling (DDPM)~\cite{ho2020denoising}, so we use the two terms interchangeably. 

\subsection{Mathematical foundation for DDPM}\label{sec:math_foundation}


The foundation for SGM stems from two mathematical observations. Firstly, a diffusion type partial differential equation (PDE) drives an arbitrary distribution to a Gaussian distribution, forming a bridge between the complex target distribution to the standard Gaussian, an easy-to-sample distribution. Secondly, such diffusion process can be simulated by its samples, translating the complicated PDE to a set of stochastic differential equations (SDEs) that are computationally easy to manipulate.

More precisely, denote $p_t(x)$ the solution to the PDE:
\begin{equation}\label{eqn:FP}
\partial_t \pf_t = \nabla\cdot(x\pf_t)+\Delta \pf_t\,.
\end{equation}
It can be shown that, for \textit{arbitrary} initial data $\pf_0$, when $T$ is big enough,
\[
\pf_T\approx\lim_{t\to\infty}\pf_t=\pi^d\,,
\]
and the convergence is exponentially fast \cite{bakry2014analysis}. 
In our context, we set the initial data $\pf_0=\ptarget$, the to-be-learned target distribution. 

This PDE can be run backward in time. 
Denote $\pb_t = \pf_{T-t}$, a quick calculation shows
\begin{equation}\label{eqn:backward_FP}
\partial_t\pb_t = -\nabla \cdot((x+2 \nabla\ln \pf_{T-t})\pb_t)+\Delta \pb_t\,.
\end{equation}
This means with the full knowledge of $\nabla \ln \pf_{T-t}$, the flow field $x+2 \nabla\ln \pf_{T-t}(x)=x+2u(T-t,x)$ drives the standard Gaussian ($\pb_0=\pf_T\approx\pi^d$) back to its original distribution, the target $\pb_T=\pf_0=\ptarget$. The term $u(t,x)=\nabla\ln \pf_{t}(x)$ is called the \textit{score function}.

Simulating these two PDEs~\eqref{eqn:FP} and~\eqref{eqn:backward_FP} directly is computationally infeasible, especially when dimension $d\gg 1$, but both equations can be represented by samples whose dynamics satisfy the corresponding SDEs. In particular, letting
\begin{equation}\label{eqn: true forward process}
    d X_t^{\rightarrow} = -X_t^{\rightarrow} dt + \sqrt{2} dB_t\,,
\end{equation}
the standard OU process, and
\begin{equation}\label{eqn: true reverse process}
    d X_t^{\leftarrow} = \left[X_t^{\leftarrow}+2u(T-t,X_t^{\leftarrow})\right] dt + \sqrt{2} dB_t',
\end{equation}
where $B_t$ and $B_t'$ are two Brownian motions, then, with proper initial conditions: \[
\mathrm{Law}(X_t^{\leftarrow}) = \pb_{t}=\pf_{T-t}=\textrm{Law}(X_{T-t}^{\rightarrow})\,.
\]
This relation translates directly simulating two PDEs~\eqref{eqn:FP} and~\eqref{eqn:backward_FP} to running its representative samples governed by SDEs~\eqref{eqn: true forward process}-\eqref{eqn: true reverse process}, significantly reducing the computational complexity. It is worth noting that if one draws $X_{t=0}^{\leftarrow}\sim \pf_T$ and runs~\eqref{eqn: true reverse process}, then:
\[
\mathrm{Law}(X_T^{\leftarrow}) = \ptarget\,,
\]
meaning the dynamics of~\eqref{eqn: true reverse process} returns a sample from the target distribution $\ptarget$, achieving the task of sampling. Here the notation $\sim$ stands for drawing an i.i.d. sample from.

\subsection{Score-function, explicit solution and score matching}\label{sec:score_match}
It is clear the success of SGM, being able to draw a sample from the target distribution $\ptarget$, lies in finding a good approximation of the score function $u(t,x)$. In the idealized setting, this score function can be explicitly expressed. In the practical computation, this function is learned from existing dataset through the score-matching procedure.

To explicitly express the score function amounts to solving~\eqref{eqn:FP}, or equivalently~\eqref{eqn: true forward process}. Taking the SDE perspective, we analyze the OU process in~\eqref{eqn: true forward process} and obtain an explicit solution:
\begin{equation}\label{eqn: mu and sigma}
    X_t^{\rightarrow} := \mu(t) y+ \sigma(t) Z\quad\text{with}\quad\begin{cases}\mu(t) := e^{-t}\\
    \sigma(t) := \sqrt{1 - e^{-2t}}\,,
    \end{cases}
\end{equation}
where $y$ is the initial data and $Z \sim \pi^d$. {Equivalently, using the PDE perspective, one sets $\pf_0=\delta_{y}$ as the initial condition to run~\eqref{eqn:FP} to form a set of Green's functions:
}
\begin{equation}\label{eqn:green}
    \pcond_t(x | y) := \mathcal{N}\left(x; \mu(t) y, \sigma(t)^2 I_{d \times d} \right)\,.
\end{equation}
These functions are Gaussian functions of $x$ centered at $\mu(t)y$ with isotropic variance $\sigma(t)^2$. This set of functions is also referred to as the transition kernel from time $0$ conditioned on $X_0^{\rightarrow} = y$ to time $t$ with $X_t^{\rightarrow} = x$.

In the idealized setting with the target distribution $\ptarget$ fully known, then with $\pf_0=\ptarget$, the solution of \eqref{eqn:FP} becomes the superposition of Green's functions weighted by $\ptarget$, namely: 
\begin{equation}\label{eqn:forward_soln}
\pf_t(x)= \int \pcond_t(x|y) \ptarget(y)dy\,,
\end{equation}
thus by definition, the score function is explicit:
\begin{equation}\label{eqn: general form for score function}
\begin{aligned}
u(t, x)&=\nabla \ln \pf_t(x) =\frac{\nabla \pf_t(x)}{\pf_t(x)}\\
&= \frac{\int u(t,x|y) \pcond_t(x|y)\ptarget(y)dy}{\int \pcond_t(x|y) \ptarget(y) dy}\,,
\end{aligned}
\end{equation}
where we called~\eqref{eqn:forward_soln} and used the notation $u(t,x|y) = \nabla\ln \pcond_t(x|y)$ to denote the conditional flow field. This function maps $\mathbb{R}_+\times\mathbb{R}^d$ to $\mathbb{R}^d$. Using the explicit formula~\eqref{eqn:green}, we have the explicit solution for the conditional flow field:
\begin{equation}\label{eqn:cond_velocity}
u(t, x|y)= - \frac{x - \mu(t) y}{\sigma(t)^2}\,.
\end{equation}
{It is a linear function on $x$ with Lipschitz constant $\frac{1}{\sigma(t)^2}$ that blows up at $t=0$.}

\noindent \textbf{Score matching.} The practical setting is not idealized: The lack of explicit formulation $p_\ast$ prevents direct computation of~\eqref{eqn: general form for score function}. Algorithmically, one needs to learn $u(t,x)$ from existing samples. A neural network (NN) is then deployed.

Intuitively, the NN should provide a function as close as possible to the true score function, meaning it solves:
\begin{equation*}
    \min_{s \in \mathcal{F}} \;\mathcal{L}_{\text{SM}}(s) := \E_{t, x} \left[ \left\| s(t,x) - u(t,x) \right\|^2 \right],
\end{equation*}
where $t \sim U[0,T]$, the uniform distribution over the time interval, and $x \sim \pf_t(x)$. $\mathcal{F}$ is a hypothesis space, and in this context, the function space representable by a class of neural networks. However, neither $\pf_t$ nor $u(t,x)$ is known in the formulation, so we turn to an equivalent problem:
\begin{equation*}
    \min_{s \in \mathcal{F}} \; \mathcal{L}_{\text{CSM}}(s) := \E_{t , y, x} \left[ \left\|s(t,x) - u(t,x|y) \right\|^2 \right],
\end{equation*}
where $t \sim U[0,T]$, $y \sim \ptarget$ and $x \sim \pcond_t(x|y)$. The subindex CSM stands for conditional-score-matching. The two problems can be shown to be mathematically equivalent, see Lemma \ref{lemma: eqv of loss SM and loss CSM}. Practically, however, this new problem is much more tractable, now with both $\pf_t(x|y)$ and $u(t,x|y)$ explicit, see~\eqref{eqn:green} and~\eqref{eqn: general form for score function}.

The target distribution $\ptarget$ is still unknown. At hands, we have many samples drawn from it: $\{y_i\}_{i=1}^N$. This allows us to reformulate the problem into an empirical risk minimization (ERM) problem:
\begin{equation}\label{eqn: ERM problem}
    \min_{s \in \mathcal{F}} \; \mathcal{L}^N_{\text{CSM}} (s) := \frac{1}{N} \sum_{i=1}^N \E_{t,x} \left[ \left\| s(t,x) - u(t,x|y_i)\right\|^2 \right]
\end{equation}
with $t \sim U[0,T]$ and $x \sim \pcond_t(x|y_i)$.

In the execution of a practical DDPM algorithm,~\eqref{eqn: ERM problem} is first run to find an NN serving as a good approximation to the score function, termed $s(t,x)$, and the user end then deploys this $s(t,x)$ in~\eqref{eqn: true reverse process} in place of $u(t,x)$ for generating a new sample from $\ptarget$. Sample $\bar{X}_0^{\leftarrow} \sim \pi^d$ and run:
\begin{equation}\label{eqn: implementable DDPM process}
    d\bar{X}_t^{\leftarrow} = \left(\bar{X}_t^{\leftarrow} + 2 s(T-t, \bar{X}_t^{\leftarrow}) \right)dt + \sqrt{2}dB_t\,.
\end{equation}
The law is denoted to be $\pbs_t := \text{Law}(\bar{X}_t^{\leftarrow})$. We note two differences comparing~\eqref{eqn: true reverse process} and \eqref{eqn: implementable DDPM process}: the initial data $\pf_T$ is replaced by $\pi^d$ and the score function $u(t,x)$ is replaced by the empirically learned score function $s(t,x)$. If both approximations are accurate, we expect $\pbs_{t}\approx\pb_t$ for all $t$.

{When minimizing the objective \eqref{eqn: ERM problem}, noting the singularity at $t=0$ of $u(t,x|y_i)$ as in~\eqref{eqn:cond_velocity},} it is a standard practice to conduct ``early stopping''~\cite{song2020score}. 
This is to take out a small fraction around the origin of time in the training~\eqref{eqn: ERM problem} and learn the score with $t\sim U[\delta,T]$.
Consequently, the sampling is also only ran up to $T-\delta$. The algorithm returns samples $\bar{X}_{T-\delta}^{\leftarrow}$ drawn from $\pbs_{T-\delta}$. The hope is $\pbs_{T-\delta}$ approximates the target $\ptarget$ using the following approximation chain:
\begin{equation*}\label{eqn:approximation_chain}
\underbrace{\pbs_{T-\delta}\approx \pb_{T-\delta}}_{\text{if}\; s\approx u\,,\;\pi^d\approx \pf_T}=\underbrace{\pf_\delta\approx \pf_0}_{\text{if}\;\delta\to0}=\ptarget\,.    
\end{equation*}

\subsection{Error analysis for DDPM}
In the idealized setting, $T\to\infty$, $s(t,x)=u(t,x)$, $\delta\to0$, and backward SDE~\eqref{eqn: implementable DDPM process} is run perfectly, then the sample initially drawn from Gaussian $\pi^d$ will represents the target distribution $\ptarget$ at $T$. Computationally, these assumptions all break: all four factors, finite $T$, nontrivial $\delta$, imperfect $s(t,x)$ and discretization error of~\eqref{eqn: implementable DDPM process} induce error. These errors were beautifully analyzed in~\cite{chen2022sampling,benton2023linear}. We summarize their results briefly.

All analysis require the target distribution to have bounded second moment.
\begin{assumption}[bounded second moment]\label{assum: bounded second moment}
We assume that $\mathfrak{m}_2^2 := \E_{y \sim p_\ast} \left[\|y\|^2 \right] < \infty$.
\end{assumption}
The learned score function is also assumed to be close to the ground-truth in $L_2(dt,\pf_tdx)$:
\begin{assumption}[score estimation error]\label{assum: score estimation error}
The score estimate $s(x,t)$ satisfies
\begin{equation*}
    \E_{t \sim U[\delta, T], x \sim \pf_t} \left[ \left\| s(t,x) - u(t,x) \right\|^2 \right] \leq \epsscore^2\,.
\end{equation*}
\end{assumption}

{Under these assumptions, it was concluded DDPM samples well:}
\begin{theorem}[Modified version of Theorem 1 in \cite{benton2023linear}]\label{theorem: convergence of DDPM}
Suppose the Assumptions \ref{assum: bounded second moment} and \ref{assum: score estimation error} hold and $T \geq 1$, $\delta > 0$.
Let $\pbs_{T-\delta}$ be the output of the DDPM algorithm \eqref{eqn: implementable DDPM process} at time $T-\delta$.
Then it holds that
\begin{equation*}
{\TV}\left(\pbs_{T-\delta}, \pf_{\delta} \right) \lesssim \epsscore + \sqrt{d} \exp(-T)
\end{equation*}
\end{theorem}
The discretization error in the original result is irrelevant to the discussion here and is omitted. This upper error bound consists of two parts.
The first term $\varepsilon_{\text{score}}$ comes from the score approximation error, while the second term $\sqrt{d}\exp(-T)$ comes from the finite truncation, where we forcefully replace $p_T$ by $\pi^d$.

The theorem states that, when $T$ is large enough and the score function is approximated well in $L_2(dt,\pf_tdx)$ sense, the TV distance between the law of generated samples $\pbs_{T-\delta}$ and $\pf_\delta\approx \ptarget$ is very small, concluding that DDPM is a good sampling strategy. 

It is tempting to further this statement and claim that DDPM is also a good generative model. Indeed, on the surface, it is typically claimed that generative models are equivalent to drawing samples from a target distribution $\ptarget$. However, we should note a stark difference between sampling and generation: A meaningful generative model should be able to produce samples that are not mere replica of known ones. The error bound in Theorem~\ref{theorem: convergence of DDPM} does not exclude this possibility. As will be shown in Section~\ref{sec: approximation error}, it is possible to design a DDPM whose score function is learned well, so according to Theorem~\ref{theorem: convergence of DDPM} produces a distribution close to the target. Yet in Section~\ref{sec: empirical optima and memorization}, we demonstrate that this model fails to be produce new samples. These two sections combined suggest DDPM with a well-learned score function does not necessarily produce a meaningful generative model.



\section{A good score estimate: sample complexity analysis}\label{sec: approximation error}
Inspired by Theorem~\ref{theorem: convergence of DDPM}, we are to design a DDPM whose learned score function satisfies Assumption~\ref{assum: score estimation error}. Throughout the section, we assume the hypothesis space is large enough ($\mathcal{F} \supseteq L^2([0,T] \times \R^d)$, for example), and the learned score estimate achieves the global minimum of the ERM~\eqref{eqn: ERM problem}. In practical training, the error heavily depends on the specific NN structure utilized in the optimization. The approximation error of the NN training is beyond the discussion point of the current paper.

Noting the objective $\mathcal{L}_{\text{CSM}}^N(s)$ is a convex functional of $s$, the optimizer has a closed-form. As derived in Proposition \ref{lemma: optimum solution to ERM problem}, for $(t,x) \in [0,T]\times \R^d$, the \textit{empirical optimal score function} is:
\begin{equation}\label{eqn: empirical optimal score function formula}
    \scoreNy(t,x) := \frac{\sum_{i=1}^N u(t,x|y_i) p_t(x|y_i)}{\sum_{j=1}^N p_t(x|y_j)}\,,
\end{equation}
where $u(t,x|y)$ is the conditional flow field, see~\eqref{eqn:cond_velocity}.

Accordingly, the DDPM draws an initial data from $\widehat{X}_0^{\leftarrow} \sim \pi^d$ and evolves the following SDE:
\begin{equation}\label{eqn: DDPM with empirical optimal score function}
d\widehat{X}_t^{\leftarrow} = (\widehat{X}_t^{\leftarrow} + 2 \scoreNy(T-t, \widehat{X}_t^{\leftarrow}) )dt + \sqrt{2}dB_t\,.
\end{equation}
We denote the law of samples $\pbsN_t := \text{Law}(\widehat{X}_t^{\leftarrow})$. The choice of the font indicates the law is produced by a finite dimensional object $\scoreNy$.

To understand the empirical optimal score function, we compare~\eqref{eqn: empirical optimal score function formula} with the ground-truth score function \eqref{eqn: general form for score function}. It is clear $\scoreNy$ can be interpreted as a Monte-Carlo (MC) sampling of $u(t,x)$, replacing both integrals in the numerator and the denominator in~\eqref{eqn: general form for score function} by empirical means. The law of large number suggests the empirical mean should converge to the true mean when the number of samples is big. Therefore, it is expected $\scoreNy$ approximates $u$ well with a very high probability when $N\gg 1$. We formulate this result in the following theorem. 
\begin{theorem}[Approximation error of empirical optimal score function]\label{thm: approximation error of empirical optimal score function}
Let $\{y_i\}_{i=1}^N$ to be $N$ i.i.d samples drawn from the target data distribution $\ptarget$.
Denote $u(t,x)$ and $\scoreNy(t,x)$ the true and empirical optimal score function, respectively, as defined in \eqref{eqn: general form for score function} and \eqref{eqn: empirical optimal score function formula}.
Then for any fixed $0 < \delta < T < \infty$, $\varepsilon_{\text{score}} > 0$ and $\tau > 0$, we have
\begin{equation*}
\E_{t \sim U[\delta, T], x \sim p_t} \left[ \left\| \scoreNy(t,x) - u(t,x) \right\|^2 \right] \leq \varepsilon_{\text{score}}^2,
\end{equation*}
with probability at least $1 - \tau$ provided that the number of training samples $N \geq N(\epsscore, \delta, \tau)$, in particular
\begin{itemize}
    \item Case 1: If $\ptarget$ is an isotropic Gaussian, i.e. $\ptarget(y) = \mathcal{N}(y; \mu_{\ptarget}, \sigma_{\ptarget}^2 I_{d\times d})$, with second moment $\mathfrak{m}_2^2 = O(d)$, then $N(\epsscore, \delta, \tau) = \frac{1}{\tau \epsscore^2} \frac{O(d)}{(2\delta)^{(d+4)/2}}$;

    \item Case 2: If $\ptarget$ is supported on the Euclidean ball of radius $R$ such that $R^2 = O(d)$, then $N(\epsscore, \delta, \tau) = \frac{1}{\tau \epsscore^2} \exp\left( \frac{O(d)}{\delta} \right)$.
\end{itemize}
\end{theorem}

The theorem implies that when the sample size is large with $N \geq N(\epsscore,\delta,\tau)$, we have high confidence, $1-\tau$, to state that the empirical optimal score function $\scoreNy$, computed using the i.i.d. samples $\{y_i\}$, is within $\epsscore$ distance from the true score function $u(t,x)$ in $L_2(dt,\pf_tdx)$.

\begin{remark}\label{rmk:sample}
A few comments are in line:
\begin{itemize}
\item[(a)]{Second moment $\mathfrak{m}_2^2 = O(d)$ and support radius $R^2=O(d)$:} The second moment and support radius being the same order as $d$ is only for notational convenience. 
In the proof, the assumption can be relaxed. When we do so, the success rate needs to be adjusted accordingly (see the discussions in Appendix \ref{appendix: approximation error}).
\item[(b)]{Implication on DDPM performance:} Combining Theorem~\ref{thm: approximation error of empirical optimal score function} with Theorem~\ref{theorem: convergence of DDPM}, it is straightforward to draw a conclusion on the performance of DDPM in terms of sample complexity. Under the same assumptions in Theorem~\ref{thm: approximation error of empirical optimal score function}, for any tolerance error $\varepsilon > 0$, by choosing $T = \log \frac{\sqrt{d}}{\varepsilon}$, $N \geq N(\varepsilon,\delta,\tau)$, then it holds that, the DDPM algorithm ran according to~\eqref{eqn: DDPM with empirical optimal score function} with the empirical optimal score function $s^N$ computed from~\eqref{eqn: empirical optimal score function formula} gives:
\begin{equation*}\label{eqn:TV_DDPM_real}
    \TV(\pbsN_{T-\delta}, \pf_{\delta}) \lesssim \varepsilon
\end{equation*}
with probability at least $1 - \tau$.
\item[(c)]{Error dependence on parameters:} Both the confidence level parameter $\tau$ and the accuracy parameter $\epsscore$ appears algebraically in $N(\epsscore, \delta, \delta)$. The rate of $\epsscore^{-2}$ comes from MC sampling convergence of $\frac{1}{\sqrt{N}}$ and is expected to be the optimal one. The rate of $\tau^{-1}$ reflects the fact that the proof uses the simple Markov inequality.
\end{itemize}
\end{remark}

We leave the main proof to Appendix \ref{appendix: approximation error} and only briefly discuss the proof strategy using Case $2$ as an example.
\begin{proof}[Sketch of proof]
Denote the error term
\begin{equation}\label{eqn:def_error_t}
\left|\errorNyt\right|^2= {\E_{x\sim p_t} \left[\left\|\scoreNy(t,x) - u(t,x) \right\|^2\right]}    
\end{equation}
and
\[
\left|\errorNy\right|^2 = {\E_{t\sim U[\delta,T]}\left|\errorNyt\right|^2}=\frac{1}{T-\delta}\int_\delta^T \left|\errorNyt\right|^2 dt\,.
\]
$\errorNy$ defines a function that maps $\{y_i\}\in\R^{Nd}$ to $\R^+$, and is a random variable itself. According to the Markov's inequality:
\begin{equation}\label{eqn: markov ineq}
\mathbb{P} \left(\errorNy > \epsscore \right)\leq \frac{\E_{\{y_i\}\sim \ptarget^{\otimes N}} \left|\errorNy\right|^2}{\epsscore^2}\,.
\end{equation}
To compute the right hand side, we note
\begin{equation}\label{eqn: score approximation error full error}  
\E_{\{y_i\}\sim \ptarget^{\otimes N}} \left|\errorNy\right|^2 = \E_{t, \{y_i\}\sim\ptarget^{\otimes N}} \left|\errorNyt\right|^2\,,
\end{equation}
and for fixed $t\in[\delta,T]$, according to the definition~\eqref{eqn:def_error_t}, one can show:
\begin{equation}\label{eqn:error_t}
\E_{\{y_i\}\sim \ptarget^{\otimes N}} \left|\errorNyt\right|^2\lesssim \frac{1}{N} \frac{1}{t} \exp\left(\frac{O(d)}{t} \right)\,.    
\end{equation}
Taking expectation with respect to $t$ in $[\delta, T]$, we have
\begin{equation*}\label{eqn: score approximation error}
\E_{\{y_i\}\sim \ptarget^{\otimes N}} \left|\errorNy\right|^2 \lesssim \frac{1}{N} \exp\left( \frac{O(d)}{\delta} \right)\,,
\end{equation*}
finishing the proof when combined with~\eqref{eqn: markov ineq}.
\end{proof}

It is clear the entire proof is built upon a direct use of the Markov inequality, and the most technical component of the proof is to give an estimate to the mean of the error term $|\errorNyt|^2$ in~\eqref{eqn:error_t}. We provide this estimate in Lemma~\ref{lemma: upper bound of taking expectation w.r.t x and y}.

\section{A bad SGM: memorization Effects}\label{sec: empirical optima and memorization}
Results in Theorem~\ref{theorem: convergence of DDPM} and Theorem~\ref{thm: approximation error of empirical optimal score function} combined implies that the DDPM \eqref{eqn: implementable DDPM process} ran with the empirical optimal score function $\scoreNy$ provides a good sampling method with a high probability. It is tempting to further this statement and call it a good generative model. We are to show in this section that this is not the case. In particular, we claim DDPM ran by $\scoreNy$ will lead to a kernel density estimation (KDE).

To be more precise, with $\{y_i\}_{i=1}^N$ i.i.d drawn from the target distribution $\ptarget$, DDPM~\eqref{eqn: implementable DDPM process} ran with $\scoreNy$ produces a distribution that is a convolution of a Gaussian with $\ptargetN=\frac{1}{N}\sum_{i=1}^N \delta_{y_i}$, and hence becomes a KDE of $\ptarget$. Since the context is clear, throughout the section we drop the lower index $\{y_i\}$ in $\scoreNy$.


The statement above stems from the following two simple observations. Firstly, the solution to the system~\eqref{eqn:FP} with initial distribution set to be $\ptargetN$ is a simple Gaussian convolution with $\ptargetN$; and secondly, the exact score function for this new system (initialized at $\ptargetN$) happens to be the empirical optimal score function~\eqref{eqn: empirical optimal score function formula}.

To expand on it, we first set the initial data for~\eqref{eqn:FP} as $\ptargetN$, the empirical distribution. Theory in Section~\ref{sec:score_match} still applies. In particular, the solution to~\eqref{eqn:FP} , denoted by $\pfN_t$, and the solution to~\eqref{eqn:backward_FP}, denoted by $\pbN_t$, still have explicit forms using the Green's functions:
\begin{equation}\label{eqn:pfNt}
\begin{aligned}
\pfN_t(x) =\pbN_{T-t}(x) &=\int\pcond_t(x|y)\ptargetN(y)dy=\frac{1}{N}\sum_{i=1}^N p_t(x|y_i)\\
&= \frac{1}{N}\sum_{i=1}^N \mathcal{N}\left(x; \mu(t) y_i, \sigma(t)^2 I_{d\times d} \right)\,.
\end{aligned}
\end{equation}
For small $t$, $\mu(t)\approx 1$ and $\sigma(t)\approx 0$, the PDE solution~\eqref{eqn:pfNt} presents a strong similarity to a KDE of $\ptarget$ with parameter $\gamma=\sigma(t)$:
\[
\ptargetNgamma(x):=\ptargetN\ast\mathcal{N}(0,\gamma^2) = \frac{1}{N} \sum_{i=1}^N \mathcal{N} \left(x; y_i, \gamma^2 I_{d\times d} \right)\,,
\]
where $\ast$ is the convolution operator. The resemblance can be characterized mathematically precisely:
\begin{proposition}\label{prop: OU with empirical and kde}
Suppose the training samples $\{y_i\}_{i=1}^N$ satisfy $\|y_i\|_2 \leq d$, for $\delta \geq 0$, $\TV(\pbN_{T-\delta}, \ptargetNgamma) \leq  \frac{d\sqrt{\delta}}{2}$ with $\gamma=\sigma(\delta)$, where $\sigma(\cdot)$ is defined in~\eqref{eqn: mu and sigma}.
\end{proposition}

{This means the forward and backward procedure described in~\eqref{eqn:FP}-\eqref{eqn:backward_FP} approximately provides a simple KDE to the target distribution when initialized with the empirical distribution.}

We now further claim this forward and backward procedure is realized by running SGM using the empirical optimal score $\scoreN$. 
To see this, we follow the computation in~\eqref{eqn: general form for score function}, and call~\eqref{eqn:pfNt} to obtain:
\begin{equation*}
\nabla\ln\pfN_t(x) =\frac{\sum_{i=1}^N \nabla \pcond_t(x|y_i)}{\sum_{j=1}^N \pcond_t(x|y_j)} = \frac{\sum_{i=1}^N \score(t,x|y_i) \pcond_t(x|y_i)}{\sum_{j=1}^N \pcond_t(x|y_i)}\,.
\end{equation*}
This means the exact score function for the KDE approximation $\pfN_t = \pbN_{T-t}$ exactly recovers $\scoreN$, the empirical optimal score for $\pf_t$, and thus SGM with empirical optimal score realizes the KDE approximation, as seen in the following proposition.

\begin{proposition}\label{prop: DDPM with empirical optimal and OU with empirical}
Under the same assumptions are in Proposition \ref{prop: OU with empirical and kde}, on the time interval $t \in [0,T]$, the total variation between the output distribution of SGM algorithm \eqref{eqn: DDPM with empirical optimal score function} with the empirical optimal score function $\pbsN_t$ and the KDE approximation $\pbN_t$ -- is bounded by $\TV\left( \pbsN_{t}, \pbN_{t}\right) \leq \frac{d}{2} \exp(-T)$.
\end{proposition}

Combine Propositions \ref{prop: OU with empirical and kde} and \ref{prop: DDPM with empirical optimal and OU with empirical} using triangle inequality, we see $\pbsN_t$ is essentially a kernel density estimation when $t$ approaches $T$. Furthermore, if one pushes $t=T\to +\infty$, we obtain the finite-support result:
\begin{theorem}[SGM with empirical optimal score function resembles KDE]\label{thm: DDPM resemble KDE}
Under the same assumptions as Proposition \ref{prop: DDPM with empirical optimal and OU with empirical}, SGM algorithm \eqref{eqn: DDPM with empirical optimal score function} with the  empirical optimal score function $\scoreN$ returns a simple Gaussian convolution with the empirical distribution in the form of~\eqref{eqn:pfNt}, and it presents the following behavior:
\begin{itemize}
   \item (\textbf{with early stopping}) for any $\varepsilon > 0$, set $T = \log \frac{d}{\varepsilon}$ and $\delta = \frac{\varepsilon^2}{d}$, we have 
   \begin{equation*}\label{eqn:TV_DDPM_KDE}
   \TV(\pbsN_{T-\delta}, \ptargetNgamma) \leq \varepsilon\,,\quad\text{with}\quad \gamma = \sigma(\delta)\,,
   \end{equation*}
    \item (\textbf{without early stopping}) by taking the limit 
    $T \rightarrow +\infty$ and $\delta = 0$, we have $\pbsN_{\infty} = \ptargetN = \frac{1}{N} \sum_{i=1}^N \delta_{y_i}$.
\end{itemize}
\end{theorem}
The theorem suggests DDPM with empirical optimal score function $\scoreN$ is, in the end, simply a KDE of the target $\ptarget$. However close KDE $\ptargetNgamma$ is to the target $\ptarget$, it is nevertheless only an object with finite amount of information.

Unlike drawing from $\ptarget$ where one can generate a completely new sample independent of the training samples, drawing from $\ptargetNgamma$ can only provide replicas of $y_i$ (with a slight shift and polluted with Gaussian noise). As a summary, SGM ran by the empirical optimal score function \underline{fails} the task of generation.

Some mathematical comments are in line. We first note that~\eqref{eqn:TV_DDPM_KDE} does not contradict~\eqref{eqn:TV_DDPM_real}. Indeed, with high probability, $\pbsN_{T-\delta}$ approximates both $\pf_\delta$ and the KDE $\ptargetNgamma$. The second bullet point (without early stopping) was also discussed in~\cite{gu2023memorization}. Our result generalize theirs to any small time $T-\delta$.

\section{Numerical Experiments}
This section is dedicated to providing numerical evidence for Theorem \ref{thm: approximation error of empirical optimal score function} and Theorem \ref{thm: DDPM resemble KDE}. Throughout the experiment, we choose the target data distribution $\ptarget$ to be a $2$-dimensional isotropic Gaussian, denoted by $\ptarget(x) = \mathcal{N}(x; \mu_{\ptarget}, \sigma_{\ptarget}^2 I_{2 \times 2})$.
The implementation details are provided in Appendix \ref{appendix: numerical experiments}.

We first estimate the score approximation error of the empirical optimal score function, as delineated in \eqref{eqn: score approximation error}, for various size of training sample $N$. Figure~\ref{fig: score approximation error} shows that the error has decreasing rate approximately $O(\frac{1}{N})$, confirming the theoretical finding in Theorem \ref{thm: approximation error of empirical optimal score function}, see also Remark~\ref{rmk:sample}(c).

\begin{figure}[!htb]
\centering
\includegraphics[width=0.45\textwidth]{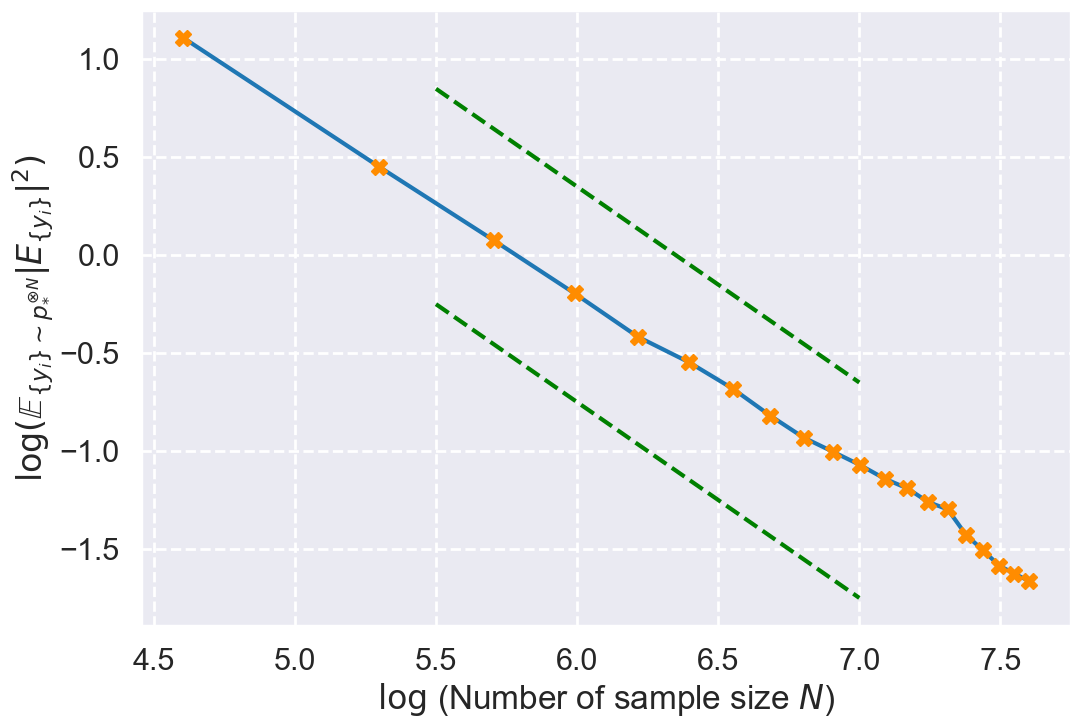}
\caption{Score approximation error of the empirical optimal score function defined in \eqref{eqn: score approximation error full error} versus the number of training samples $N$.
Both $x$-axis and $y$-axis are in the logarithmic scales.
The orange crosses represent the score approximation error for varying values of $N$, with a fitted blue trend line.
Reference lines with a slope of $-1$ are depicted by the green dashed lines, illustrating that the slope of the blue line is also approximately $-1$.
This observation corroborates the rate $O(\frac{1}{N})$ provided in Theorem \ref{thm: approximation error of empirical optimal score function}.}
\label{fig: score approximation error}
\end{figure}

Secondly, we showcase Theorem \ref{thm: DDPM resemble KDE} and demonstrate that DDPM behaves as a KDE when equipped with empirical optimal score function. As seen in Figure~\ref{fig: generated samples 2d-gaussian}, samples produced by DDPM ran with $s^N$ exhibit a high concentration around the training samples. Conversely, while the samples generated by DDPM ran with the true score function $u(t,x)$ appear to be drawn from the same distribution as the training samples, they are not mere duplicates of the existing ones.

\begin{figure}[!htb]
    \centering
    \begin{subfigure}{.25\textwidth}%
    \includegraphics[width=1.0\linewidth]{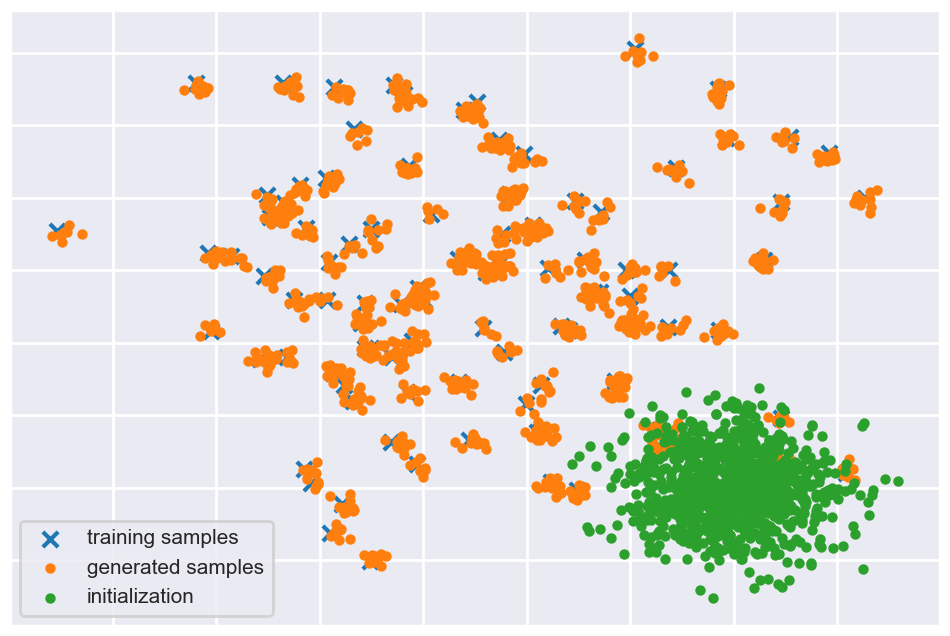}
    \end{subfigure}%
    \begin{subfigure}{.25\textwidth}%
    \includegraphics[width=1.0\linewidth]{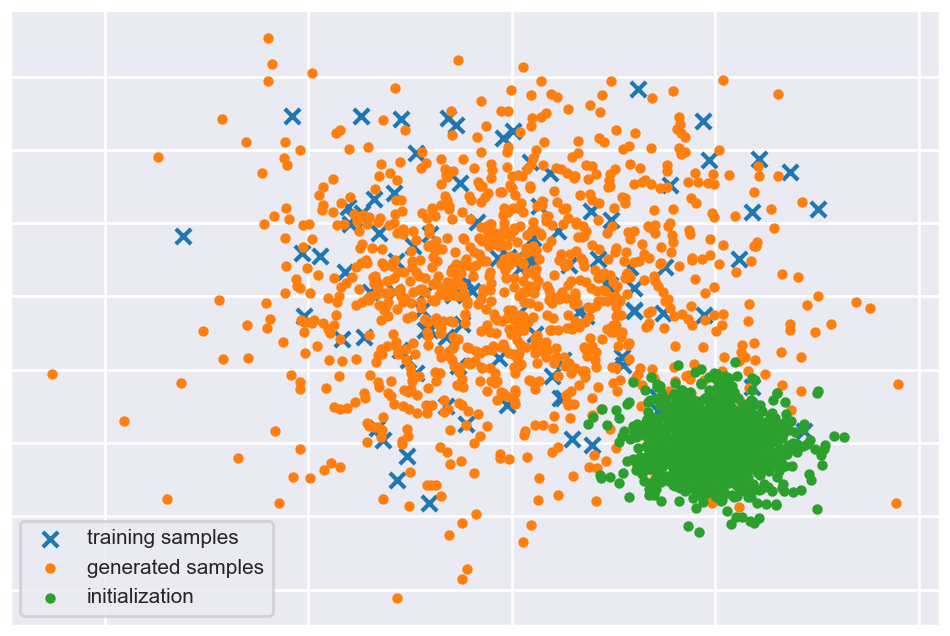}
    \end{subfigure}
\caption{\textbf{Left:} Samples generated by DDPM with \textit{empirical optimal score function} $s^N(t,x)$.  \textbf{Right:} Samples generated by DDPM with \textit{true score function} $u(t,x)$.
In both plots, the blue crosses are the training samples, the green dots are the initialization positions and the orange dots are the outputs of DDPM with early stop of $\delta=0.01$.
}
\label{fig: generated samples 2d-gaussian}
\end{figure}

\section{Discussion and Conclusion}

The classical theory measures the success of score-based generative model based on the distance of the learned distribution and the ground-truth distribution. Under this criterion, SGM would be successful if the score function is learned well. 

In this paper, we provide a counter-example of SGM that has a good score approximation while produces meaningless samples. On one hand, the application of Theorem \ref{theorem: convergence of DDPM} and Theorem \ref{thm: approximation error of empirical optimal score function} combined suggest SGM equipped with empirical optimal score function learns a distribution close to the ground-truth. On the other hand, Theorem \ref{thm: DDPM resemble KDE} suggests this scenario resembles the Gaussian kernel density estimation and can only generate existing training samples with Gaussian blurring.

This apparent paradox between sound theoretical convergence and poor empirical new sample generations indicates that current theoretical criteria may not be sufficient to fully evaluate the performance of generative models. It strongly focuses on the ``imitation'' capability and losses out on quantifying ``creativity''. 
Similar features were presented in other generative models like generative adversarial networks \cite{vardanyan2023guaranteed}, and different criteria have been proposed~\cite{vardanyan2023guaranteed,yi2023generalization}, yet a comprehensive end-to-end convergence analysis for these criteria has not been done for SGMs.
We leave this exploration to future research.

\section*{Broader Impact}
Our results, while being theoretical in nature, have potential positive impacts in motivating better frameworks to ensure that the generative model do not create unintended leakage of private information.
We believe that there are no clear negative societal consequences of this theoretical work.




\section*{Acknowledgements}
The three authors are supported in part by NSF-DMS 1750488, and NSF-DMS 2308440. S.~Li is further supported by NSF-DMS 2236447.


\bibliography{ref}
\bibliographystyle{icml2024}

\newpage
\appendix
\onecolumn

\section{Notations}\label{appendix: notations}
\textbf{Partial differential equations (PDEs).}
Let $\R^d$ to be the $d$-dimensional Euclidean space and $T > 0$ is the time horizon.
Denote $x = (x_1, x_2, \dots, x_d)^T \in \R^d$ and $t \in [0,T]$ to be the spatial variable and time variable respectively.
The gradient of a real-valued function $p$ with respect to the spatial variable and the time-derivative of $p$ are denoted by $\nabla p = \left(\frac{\partial p}{\partial x_1}, \frac{\partial p}{\partial x_2}, \cdots, \frac{\partial p}{\partial x_d}\right)$ and $\partial_t p$ respectively.
The Laplacian of $p$ is denoted by $\Delta p = \nabla \cdot (\nabla p)$.
Here, $\nabla \cdot F = \sum_{i=1}^d \frac{\partial F_i}{\partial x_i}$ indicates the divergence of $F = (F_1, F_2, \cdots, F_d)$ with respect to the spatial variable $x$.

\textbf{Stochastic differential equations (SDEs) and their laws.}
\begin{itemize}
    \item The target data distribution is $\ptarget$.
    
    \item The forward process \eqref{eqn: true forward process} initialized at the target distribution $\ptarget$ is denoted $(X_t^{\rightarrow})_{t \in [0,T]}$, and $\pf_t := \mathrm{Law}(X_t^{\rightarrow})$.
    
    \item The backward process \eqref{eqn: true reverse process} is denoted $(X_t^{\leftarrow})_{t \in [0,T]}$, where $\mathrm{Law}(X_t^{\leftarrow}) := \pb_{t}=\pf_{T-t}=\textrm{Law}(X_{T-t}^{\rightarrow})$.
    
    \item The DDPM algorithm \eqref{eqn: implementable DDPM process} with arbitrary learned score function is denoted $(\bar{X}_t^{\leftarrow})_{t \in [0,T]}$ and $\pbs_t := \textrm{Law}(\bar{X}_t^{\leftarrow})$.
    We initialize the process at $\pbs_0 = \pi^d$, the standard Gaussian distribution.
    
    \item The DDPM algorithm \eqref{eqn: DDPM with empirical optimal score function} with the empirical optimal score function $s^N$ is denoted by $(\widehat{X}_t^{\leftarrow})_{t \in [0,T]}$.
    We indicate the law at time $t$ as $\pbsN_t := \textrm{Law}(\widehat{X}_t^{\leftarrow})$ and let $\pbsN_0 = \pi^d$.

     \item The law of forward process \eqref{eqn: true forward process} initialized at the empirical distribution $\ptargetN$ at time $t\in [0,T]$ is indicated by $\pfN_t$.
     The law of corresponding backward process at time $t \in [0,T]$ is denoted by $\pbN_{t} = \pfN_{T-t}$.

\end{itemize}

\textbf{Other notations.}
We denote $\ptarget$ as the target data distribution supported on a subset of $\R^d$, and indicate the empirical distribution by $\ptargetN$.
The Gaussian kernel with bandwidth $\gamma$ is denoted by $\mathcal{N}_{\gamma} := \mathcal{N}(0, \gamma^2 I_{d\times d})$.
For the special case $\gamma = 1$, i.e. standard Gaussian, we use notation $\pi^d := \mathcal{N}(0, I_{d\times d})$.
We denote the Gaussian KDE with bandwidth $\gamma$ as $\ptargetNgamma := \ptargetN * \mathcal{N}_{\gamma}$.
The early stopping time of running SDEs is indicated by $\delta \in [0,T)$.
We use $i \in [N]$ to denote $i = 1,2,\dots, N$.

\section{Empirical optimal score function}\label{appendix: empirical optimal score function}
\begin{lemma}\label{lemma: eqv of loss SM and loss CSM}
Assuming that $p_t(x) > 0$ for all $x \in \R^d$ and $t \in [0,T]$, then up to a constant independent of function $s \in L^2([0,T] \times \R^d)$, $\mathcal{L}_{\text{SM}}(s)$ and $\mathcal{L}_{\text{CSM}}(s)$ are equal.
\end{lemma}
\begin{proof}
We follow the proof of Theorem 2 in \cite{lipman2022flow}.
We assume that $\ptarget(x)$ are decreasing to zero at a sufficient speed as $\|x\| \rightarrow \infty$, and $u(t,x), s(t,x)$ are bounded in both time and space variables.
These assumptions ensure the existence of all integrals and allow the changing of integration order (by Fubini's theorem).

To prove $\mathcal{L}_{\text{SM}}(s)$ and $\mathcal{L}_{\text{CSM}}(s)$ are equal up to a constant independent of function $s$, we only need to show that for any fixed $t \in [0, T]$, 
\begin{equation*}
\E_{x \sim p_t} \left[ \left\|s(t,x) - u(t,x) \right\|^2 \right] = \E_{y \sim \ptarget, x \sim p_t(x|y)} \left[\left\|s(t,x) - u(t,x|y) \right\|^2 \right] + C,
\end{equation*}
where $C$ is a constant function that independent of function $s$.
We can compute that
\begin{equation*}
\begin{aligned}
\E_{x \sim p_t} \left[ \left\|u(t,x) \right\|^2 \right] = \int \left\|s(t,x) \right\|^2 p_t(x) dx= \int \int \left\|s(t,x) \right\|^2 p_t(x|y) \ptarget(y)dy  = \E_{y \sim \ptarget, x \sim p_t(x|y)} \left[\left\|s(t,x) \right\|^2 \right],
\end{aligned}
\end{equation*}
where the second equality we use the definition of $p_t(x)$, and in the third equality we change the order of integration.
\begin{equation*}
\begin{aligned}
\E_{x \sim p_t} \left[ \langle s(t,x), u(t,x) \rangle \right] &= \int \langle s(t,x), \frac{\int u(t,x|y) p_t(x|y) \ptarget(y)dy}{p_t(x)} \rangle p_t(x) dx\\
&= \int \langle s(t,x), \int u(t,x|y) p_t(x|y) \ptarget(y)dy \rangle dx\\
&= \int \langle s(t,x), u(t,x|y) \rangle p_t(x|y) \ptarget(y)dy dx\\
&= \E_{y \sim \ptarget, x \sim p_t(x|y)} \left[ \langle s(t,x), u(t,x|y) \rangle \right] \qquad (\text{by Fubini's theorem})
\end{aligned}
\end{equation*}
Therefor we have
\begin{equation*}
\begin{aligned}
\E_{x \sim p_t} \left[ \left\|s(t,x) - u(t,x) \right\|^2 \right] &= \E_{x \sim p_t} \left[ \left\|s(t,x) \right\|^2 \right] - 2 \E_{x \sim p_t} \left[ \langle s(t,x), u(t,x) \rangle \right] + \E_{x \sim p_t} \left[ \left\|u(t,x) \right\|^2 \right]\\
&= \E_{y \sim \ptarget, x \sim p_t(x|y)} \left[\left\|s(t,x) \right\|^2 \right] - 2 \E_{y \sim \ptarget, x \sim p_t(x|y)} \left[ \langle s(t,x), u(t,x|y) \rangle \right] + \E_{x \sim p_t} \left[ \left\|u(t,x) \right\|^2 \right]\\
&= \E_{y \sim \ptarget, x \sim p_t(x|y)} \left[\left\|s(t,x) - u(t,x|y) \right\|^2 \right] + C,
\end{aligned}
\end{equation*}
where the last inequality comes from the fact that $u(t,x)$ and $u(t,x|y)$ are independent of $s(t,x)$.
\end{proof}

\begin{lemma}\label{lemma: optimum solution to ERM problem}
The optimizer $s^N$ of the objective function 
\begin{equation*}
    \min_{s \in L^2([0,1] \times \R^d)} \mathcal{L}^N_{\text{CSM}} (s) := \frac{1}{N} \sum_{i=1}^N \E_{t\sim U[0,1], x \sim p_t(x|y_i)} \left[ \left\| s(t,x) - u(t,x|y_i)\right\|^2 \right]
\end{equation*}
has the form
\begin{equation*}
    s^N(t,x) := \frac{\sum_{i=1}^N u(t,x|y_i) p_t(x|y_i)}{\sum_{j=1}^N p_t(x|y_j)}, \qquad t \in [0,T], x \in \R^d
\end{equation*} 
\end{lemma}
\begin{proof}
Since the objective $\mathcal{L}_{\text{CSM}}(s)$ is a convex functional of $s$, by the first-order optimality condition, the optimizer $s^N$ should satisfy
\begin{equation*}
\frac{\delta \mathcal{L}_{\text{CSM}}(s)}{\delta s} \Bigg|_{s = s^N} = \frac{2}{N} \sum_{i=1}^N \left[s^N(t, x) - u(t, x| y_i) \right] p_t(x|y_i) = 0,
\end{equation*}
which implies that for $t \in [0,T], x\in \R^d$,
\begin{equation*}
    s^N(t,x) = \frac{\sum_{i=1}^N u(t,x|y_i) p_t(x|y_i)}{\sum_{j=1}^N p_t(x|y_j)}.
\end{equation*}
\end{proof}

\section{Approximation error of empirical optimal score function}\label{appendix: approximation error}
In this section, we provide the full proof of Theorem \ref{thm: approximation error of empirical optimal score function}.
For the completeness, we state the theorem again in the following:

\begin{theorem}[Approximation error of empirical optimal score function]\label{theorem: approximation error new version}
Let $\{y_i\}_{i=1}^N$ to be $N$ i.i.d samples drawn from the target data distribution $\ptarget$.
Denote $u(t,x)$ and $\scoreNy(t,x)$ the true and empirical optimal score function respectively, as defined in \eqref{eqn: general form for score function} and \eqref{eqn: empirical optimal score function formula}.
Then for any fixed $0 < \delta < T < \infty$, $\varepsilon_{\text{score}} > 0$ and $\tau > 0$, we have
\begin{equation*}
\E_{t \sim U[\delta, T], x \sim p_t} \left[ \left\| \scoreNy(t,x) - u(t,x) \right\|^2 \right] \leq \varepsilon_{\text{score}}^2,
\end{equation*}
with probability at least $1 - \tau$ provided that the number of training samples $N \geq N(\epsscore, \delta, \tau)$, where $ N(\epsscore, \delta, \tau)$ is defined based on the nature of $\ptarget$:
\begin{itemize}
    \item Case 1: If $\ptarget$ is an isotropic Gaussian, i.e. $\ptarget(y) = \mathcal{N}(y; \mu_{\ptarget}, \sigma_{\ptarget}^2 I_{d\times d})$, with second moment $\mathfrak{m}_2^2 = O(d)$, then $N(\epsscore, \delta, \tau) = \frac{1}{\tau \epsscore^2} \frac{O(d)}{(2\delta)^{(d+4)/2}}$;

    \item Case 2: If $\ptarget$ is supported on the Euclidean ball of radius $R$ such that $R^2 = O(d)$, then $N(\epsscore, \delta, \tau) = \frac{1}{\tau \epsscore^2} \exp\left( \frac{O(d)}{\delta} \right)$.
\end{itemize}
\end{theorem}
\begin{proof}[Proof of Theorem \ref{thm: approximation error of empirical optimal score function}]
Denote the error term
\begin{equation}\label{eqn: approximation error w.r.t. x}
\left|\errorNyt\right|^2= {\E_{x\sim p_t} \left[\left\|\scoreNy(t,x) - u(t,x) \right\|^2\right]}    
\end{equation}
and
\[
\left|\errorNy\right|^2 = {\E_{t\sim U[\delta,T]}\left|\errorNyt\right|^2}=\frac{1}{T-\delta}\int_\delta^T \left|\errorNyt\right|^2 dt\,.
\]
$\errorNy$ defines a function that maps $\{y_i\}\in\R^{Nd}$ to $\R^+$, and is a random variable itself. According to the Markov's inequality:
\begin{equation}\label{eqn: markov's ineq in appendix}
\mathbb{P} \left(\errorNy > \epsscore \right)\leq \frac{\E_{\{y_i\}\sim \ptarget^{\otimes N}} \left|\errorNy\right|^2}{\epsscore^2}\,.
\end{equation}
The final conclusions are mainly built on the upper bound of the right hand side in the Markov's inequality above. 
We prove the results for Case 1 and Case 2 respectively.
\begin{itemize}
    \item Case 1: Note that
\begin{equation}\label{eqn: score approximation error in appendix}
\E_{\{y_i\} \sim \ptarget^{\otimes N}} \left| \errorNy \right|^2 = \E_{t \sim U[\delta, T]} \E_{\{y_i\} \sim \ptarget^{\otimes N}, x\sim p_t} \left[ \left\|\scoreNy (t,x) - u(t,x) \right\|^2 \right], 
\end{equation}
and for fixed $t \in [\delta, T]$, according to the definition \eqref{eqn: approximation error w.r.t. x}, one can show (Lemma \ref{lemma: upper bound of taking expectation w.r.t x and y})
\begin{equation}\label{eqn: upper bound of taking expectation w.r.t x and y in Gaussian case}
    \E_{\{y_i\} \sim \ptarget^{\otimes N}} \left| \errorNyt \right|^2 \lesssim \frac{O(d)}{N \left( 1 - e^{-2t} \right)^{(d+6)/2}} \, .
\end{equation}
Taking expectation with respect to $t$ in $[\delta, T]$, we have
\begin{equation*}
    \E_{\{y_i\} \sim \ptarget^{\otimes N}} \left| \errorNy \right|^2 \lesssim \frac{1}{T - \delta} \int_{\delta}^T \frac{O(d)}{N \left( 1 - e^{-2t} \right)^{(d+6)/2}} dt \lesssim \frac{O(d)}{N} \int_{\delta}^T \frac{1}{(2t)^{(d+6)/2}}dt \lesssim \frac{1}{N} \frac{O(d)}{(2\delta)^{(d+4)/2}} \, .
\end{equation*}
By the Markov's inequality \eqref{eqn: markov's ineq in appendix}, we have
\begin{equation*}
\E_{t \sim U[\delta, T], x \sim p_t} \left[ \left\| \scoreNy(t,x) - u(t,x) \right\|^2 \right] \leq \varepsilon_{\text{score}}^2,
\end{equation*}
with probability $1 - \frac{1}{N\epsscore^2} \frac{O(d)}{(2\delta)^{(d+4)/2}}$.
Letting $\frac{1}{N\epsscore^2} \frac{O(d)}{(2\delta)^{(d+4)/2}} = \tau$, we compute the sample complexity $N(\epsscore, \delta, \tau) = \frac{1}{\tau \epsscore^2} \frac{O(d)}{(2\delta)^{(d+4)/2}}$.

\item Case 2: For fixed $t \in [\delta, T]$, according to the definition \eqref{eqn: approximation error w.r.t. x}, one can show (Lemma \ref{lemma: upper bound of taking expectation w.r.t x and y})
\begin{equation}\label{eqn: upper bound of taking expectation w.r.t x and y}
    \E_{\{y_i\} \sim \ptarget^{\otimes N}} \left| \errorNyt \right|^2 \lesssim \frac{1}{N} \frac{1}{t} \exp\left(\frac{O(d)}{t}  \right)
\end{equation}
Taking expectation with respect to $t$ in $[\delta, T]$, we have (Lemma \ref{lemma: upper bound of taking expectation w.r.t. t})
\begin{equation*}
    \E_{\{y_i\} \sim \ptarget^{\otimes N}} \left| \errorNy \right|^2 \lesssim \frac{1}{T - \delta} \int_{\delta}^T \frac{1}{N} \frac{1}{t} \exp \left( \frac{O(d)}{t} \right) dt \lesssim \frac{1}{N} \exp\left( \frac{O(d)}{\delta} \right)
\end{equation*}
Again by the Markov's inequality \eqref{eqn: markov's ineq in appendix} and similar computations in Case 1, we have the sample complexity $N(\epsscore, \delta, \tau) = \frac{1}{\tau \epsscore^2} \exp\left( \frac{O(d)}{\delta} \right)$.
\end{itemize}
\end{proof}

\begin{lemma}\label{lemma: upper bound of taking expectation w.r.t x and y}
Under the same assumptions as in Theorem \ref{theorem: approximation error new version}, for fixed $t \in [\delta, T]$, we have
\begin{itemize}
    \item Case 1: If $\ptarget$ is an isotropic Gaussian with second moment $\mathfrak{m}_2^2 = O(d)$, then
    \begin{equation*}
    \E_{\{y_i\} \sim \ptarget^{\otimes N}, x\sim p_t} \left[ \left\|\scoreNy (t,x) - u(t,x) \right\|^2 \right] \lesssim \frac{O(d)}{N \left( 1 - e^{-2t} \right)^{(d+6)/2}} \,;
\end{equation*}

    \item Case 2: If $\ptarget$ is supported on the Euclidean ball with radius $R > 0$ such that $R^2 = O(d)$, then
\begin{equation*}
    \E_{\{y_i\} \sim \ptarget^{\otimes N}, x\sim p_t} \left[ \left\|\scoreNy (t,x) - u(t,x) \right\|^2 \right] \lesssim \frac{1}{N} \frac{1}{t} \exp\left(\frac{O(d)}{t} \right)\,.
\end{equation*}
\end{itemize}
\end{lemma}
\begin{proof}
\begin{itemize}
    \item Case 1:
By the definitions of $u(t,x)$ and $\scoreNy(t,x)$ in \eqref{eqn: general form for score function} and $\eqref{eqn: empirical optimal score function formula}$, we can rewrite them as
\begin{align}
u(t,x) &= \frac{\int u(t,x|y) \pcond_t(x|y) \ptarget(y)dy}{\int \pcond_t(x|y) \ptarget(y) dy} = -\frac{1}{\sigma(t)^2} x + \frac{\mu(t)}{\sigma(t)^2} \frac{\int y \pcond_t(x|y)\ptarget(y)dy}{\int \pcond_t(x|y) \ptarget(y) dy} := a_t x + b_t \frac{v_t(x)}{\pf_t(x)},\label{eqn: true score second form}\\
s^N_{\{y_i\}}(t,x) &= \frac{\sum_{i=1}^N u(t,x|y_i) \pcond_t(x|y_i)}{\sum_{j=1}^N \pcond_t(x|y_j)} = a_t x + b_t \frac{\frac{1}{N} \sum_{i=1}^N y_i \pcond_t(x|y_i)}{\frac{1}{N} \sum_{j=1}^N \pcond_t(x|y_i)} := a_t x + b_t \frac{v_t^N(x)}{\pf_t^N(x)}, \label{eqn: empirical optimal score second form}
\end{align}
where we denote $a_t := -\frac{1}{\sigma(t)^2}$ and $b_t := \frac{\mu(t)}{\sigma(t)^2}$. Then we can compute
\begin{equation*}
\begin{aligned}
\left\| \scoreNy(t,x) - u(t,x) \right\|^2 &= \left\| \left(a_t x + b_t \frac{v_t^N(x)}{\pf_t^N(x)} \right) - \left(a_t x + b_t \frac{v_t(x)}{\pf_t(x)} \right) \right\|^2\\
&= b_t^2 \left\| \frac{1}{\pf_t(x)} \left( v_t^N(x) - v_t(x) \right) + \left( \frac{1}{\pf_t^N(x)} - \frac{1}{\pf_t(x)} \right) v_t^N(x) \right\|^2\\
&\leq 2 b_t^2 \left( \frac{1}{\pf_t(x)^2} \left\| v_t^N(x) - v_t(x)\right\|^2 + \left( \frac{\pf_t^N(x) - \pf_t(x)}{\pf_t(x)} \right)^2 \frac{\left\| v_t^N(x) \right\|^2}{\pf_t^N(x)^2} \right)\, ,
\end{aligned}
\end{equation*}
where the last inequality is the Young's. 
Then we have:
\begin{equation*}
\begin{aligned}
&\E_{\{y_i\} \sim \ptarget^{\otimes N}, x \sim p_t} \left[ \left\|s^N_{\{y_i\}}(t,x) - u(t,x) \right\|^2 \right]\\
&\leq 2b_t^2 \E_{x \sim p_t} \left[\E_{\{y_i\} \sim \ptarget^{\otimes N}}\left[\frac{1}{p_t(x)^2} \left\| v_t^N(x) - v_t(x) \right\|^2 + \left( \frac{\pf_t^N(x) - \pf_t(x)}{\pf_t(x)} \right)^2 \frac{\left\| v_t^N(x) \right\|^2}{\pf_t^N(x)^2} \right] \right]\\
&\lesssim b_t^2 \E_{x\sim p_t} \left[ \frac{\|x\|^2 + \mathfrak{m}_2^2}{N \mu(t)^2} \exp \left( \frac{\|x - \mu(t) \mu_{\ptarget}\|^2}{2 \left( \frac{(\sigma(t)^2 + \mu(t)^2 \sigma_{\ptarget}^2)(\sigma(t)^2/2 + \mu(t)^2 \sigma_{\ptarget}^2)}{\mu(t)^2 \sigma_{\ptarget}^2} \right)} \right) \right] \qquad (\text{by Lemma \ref{lemma: upper bound w.r.t y gaussian case}}) \\
&\lesssim \frac{1}{N} \frac{b_t^2}{\mu(t)^2} \displaystyle{\int} \left( \|x\|^2 + \mathfrak{m}_2^2 \right) \exp \left( \frac{\|x - \mu(t) \mu_{\ptarget}\|^2}{2 \left( \frac{(\sigma(t)^2 + \mu(t)^2 \sigma_{\ptarget}^2)(\sigma(t)^2/2 + \mu(t)^2 \sigma_{\ptarget}^2)}{\mu(t)^2 \sigma_{\ptarget}^2} \right)} \right)  \exp\left( - \frac{\|x - \mu(t)\mu_{\ptarget}\|^2}{2 (\sigma(t)^2 + \mu(t)^2 \sigma_{\ptarget}^2)} \right) dx\\
&= \frac{1}{N} \frac{b_t^2}{\mu(t)^2} \displaystyle{\int} \left( \|x\|^2 + \mathfrak{m}_2^2 \right)  \exp\left( -\frac{\|x - \mu(t)\mu_{\ptarget}\|^2}{2\left( \frac{(\sigma(t)^2 + \mu(t)^2 \sigma_{\ptarget}^2) (\sigma(t)^2 + 2\mu(t)^2 \sigma_{\ptarget}^2}{\sigma(t)^2} \right)} \right) dx\\
&\propto \frac{1}{N} \frac{b_t^2}{\mu(t)^2} \frac{1}{\sigma(t)^d} \left[ \left( \|\mu(t)\mu_{\ptarget}\|^2 + d \left( \frac{(\sigma(t)^2 + \mu(t)^2 \sigma_{\ptarget}^2) (\sigma(t)^2 + 2\mu(t)^2 \sigma_{\ptarget}^2)}{\sigma(t)^2} \right)\right) + \mathfrak{m}_2^2 \right]\\
&\lesssim \frac{1}{N} \left( \frac{\mathfrak{m}_2^2}{\sigma(t)^{d+4}} + \frac{d}{\sigma(t)^{d+6}} \right) \qquad (\text{by the definition of $b_t = \frac{\mu(t)}{\sigma(t)^2}$})\\
&\lesssim \frac{1}{N} \frac{O(d)}{\sigma(t)^{6}} = \frac{O(d)}{N (1 - e^{-2t})^{(d+6)/2}}  \qquad (\text{by $\mathfrak{m}_2^2 = O(d)$}) \, .
\end{aligned}
\end{equation*}

\item Case 2: We use the same notations as in Case 1 and define
\begin{equation*}
    A_1=\E_{x, \{y_i\}} \left[ \frac{1}{p_t(x)^2} \left\| v_t^N(x) - v_t(x)\right\|^2 \right]\,,\quad\text{and}\quad A_2=\E_{x, \{y_i\}} \left[ \left( \frac{p_t^N(x) - p_t(x)}{p_t(x)} \right)^2 \frac{\left\| v_t^N(x) \right\|^2}{p_t^N(x)^2} \right]\,.
\end{equation*}
Then we have:
\begin{equation*}
\begin{aligned}
\E_{\{y_i\} \sim \ptarget^{\otimes N}, x\sim p_t} \left[ \left\|\scoreNy (t,x) - u(t,x) \right\|^2 \right] &\leq 2b_t^2\left(A_1 + A_2 \right)\,.
\end{aligned}
\end{equation*}
We now bound terms $A_1$ and $A_2$ respectively. For term $A_1$, we have
\begin{equation*}
\begin{aligned}
A_1 &= \E_{x \sim p_t, \{y_i\}\sim \ptarget^{\otimes N}} \left[ \frac{1}{p_t(x)^2} \left\| v_t^N(x) - v_t(x)\right\|^2 \right]\\
&= \E_{x \sim p_t} \left[\frac{1}{p_t(x)^2} \E_{\{y_i\} \sim \ptarget^{\otimes N}} \left[ \left\| v_t^N(x) - v_t(x) \right\|^2 \right] \right]\\
&\leq \frac{1}{N} \E_{x \sim p_t} \left[ \frac{1}{p_t(x)^2}  \E_{y \sim \ptarget} \|y p_t(x|y)\|^2 \right] \qquad  (\text{by Lemma \ref{lemma: variance of mean type result}})\\
&= \frac{1}{N} \frac{1}{\left(2\pi \sigma(t)^2 \right)^d} \int \int \frac{1}{p_t(x)} \|y\|^2 \exp\left( -\frac{2\|x - \mu(t)y\|^2}{2 \sigma(t)^2} \right) \ptarget(y) dy dx \qquad (\text{by the definition of $p_t(x|y)$})\\
&\leq \frac{1}{N} \frac{K^{-1}_t}{\left( 2\pi \sigma(t)^2 \right)^{d/2}} \int \|y\|^2 \left( \int \exp\left( \frac{1 + \lambda \mu(t)}{2\sigma(t)^2} \|x\|^2 \right) \exp\left( -\frac{2\|x - \mu(t)y\|^2}{2 \sigma(t)^2} \right) dx  \right) \ptarget(y) dy \qquad (\text{by Lemma \ref{lemma: lower bound of p_t(x)}})\\
&= \frac{1}{N} \frac{K^{-1}_t}{\left( 2\pi \sigma(t)^2 \right)^{d/2}} \int \|y\|^2 \exp\left( \frac{\mu(t)^2(1 + \lambda \mu(t))}{\sigma(t)^2 (1 - \lambda \mu(t))} \|y\|^2 \right) \left( \int \exp\left(-\frac{1-\lambda \mu(t)}{2\sigma(t)^2} \left\|x - \frac{2\mu(t)}{1 - \lambda \mu(t)} y \right\|^2 \right)dx \right) \ptarget(y)dy\\
&= \frac{1}{N} \frac{K^{-1}_t}{(1 - \lambda \mu(t))^{d/2}} \int \|y\|^2 \exp\left( \frac{\mu(t)^2(1 + \lambda \mu(t))}{\sigma(t)^2 (1 - \lambda \mu(t))} \|y\|^2 \right) \ptarget(y)dy \\
&= \frac{1}{N} \frac{1}{(1 - \lambda \mu(t))^{d/2}} \frac{\int \|y\|^2 \exp\left( \frac{\mu(t)^2(1 + \lambda \mu(t))}{\sigma(t)^2 (1 - \lambda \mu(t))} \|y\|^2 \right) \ptarget(y)dy}{\int \exp\left(-\frac{\mu(t) + \lambda \mu(t)^2}{2 \lambda \sigma(t)^2} \|y\|^2 \right) \ptarget(y) dy} \qquad (\text{by the definition of $K_t$})\\
&\leq \frac{1}{N} \frac{1}{(1 - \lambda \mu(t))^{d/2}} R^2 \exp\left( \frac{\mu(t)^2 (1 + \lambda \mu(t))}{\sigma(t)^2 (1 - \lambda \mu(t))} R^2 \right) \exp \left( \frac{\mu(t) + \lambda \mu(t)^2}{2\lambda \sigma(t)^2} R^2 \right) \qquad (\text{by $\mathrm{supp}(\ptarget) \subseteq B(0,R)$})\\
&= \frac{1}{N} \frac{R^2}{(1 - \lambda \mu(t))^{d/2}} \exp \left( \frac{\mu(t) (1 + \lambda \mu(t))^2}{2 \lambda \sigma(t)^2 (1 - \lambda \mu(t))} R^2 \right) \\
&= \frac{1}{N} 2^{d/2} R^2 \exp \left( \frac{9\mu(t)^2}{2\sigma(t)^2} R^2 \right) \qquad (\text{by choosing $\lambda = \frac{1}{2\mu(t)}$})\\
&= \frac{1}{N} \exp\left( \frac{\mu(t)^2}{\sigma(t)^2} O(d)\right),
\end{aligned}
\end{equation*}
where we assume $R^2 = O(d)$.
For term $A_2$, we can calculate
\begin{equation*}
\begin{aligned}
A_2 &= \E_{x \sim p_t, \{y_i\} \sim \ptarget^{\otimes N}} \left[ \left( \frac{p_t^N(x) - p_t(x)}{p_t(x)} \right)^2 
 \frac{\left\| v_t^N(x) \right\|^2}{p_t^N(x)^2}\right]\\
&\leq R^2 \E_{x \sim p_t} \left[ \frac{1}{p_t(x)^2} \E_{\{y_i\} \sim \ptarget^{\otimes N}} \left( p_t^N(x) - p_t(x) \right)^2 \right] \qquad (\text{by Lemma \ref{lemma: bound for empirical division}})\\
&\leq \frac{R^2}{N} \E_{x \sim p_t} \left[ \frac{1}{p_t(x)^2} \E_{\{y_i\} \sim \ptarget^{\otimes N}} \left[p_t(x|y)^2 \right] \right] \qquad (\text{by Lemma \ref{lemma: variance of mean type result}})\\
&\leq \frac{1}{N} \frac{K^{-1}_t R^2}{(2\pi \sigma(t)^2)^{d/2}} \int \left( \int \exp\left( \frac{1 + \lambda \mu(t)}{2\sigma(t)^2} \|x\|^2 \right) \exp\left( -\frac{2\|x - \mu(t)y\|^2}{2 \sigma(t)^2} \right) dx  \right) \ptarget(y) dy \qquad (\text{by Lemma \ref{lemma: lower bound of p_t(x)}})\\
&\leq \frac{1}{N} \exp \left( \frac{\mu(t)^2}{\sigma(t)^2} O(d) \right) \qquad (\text{by the same computations as for term $A_1$})
\end{aligned}
\end{equation*}
Combining the upper bounds for terms $A_1$ and $A_2$, we obtain
\begin{equation*}
\begin{aligned}
\E_{\{y_i\} \sim \ptarget^{\otimes N}, x\sim p_t} \left[ \left\|\scoreNy (t,x) - u(t,x) \right\|^2 \right] &\lesssim \frac{1}{N} \frac{\mu(t)^2}{\sigma(t)^4} \exp \left( \frac{\mu(t)^2}{\sigma(t)^2} O(d) \right)\\
&= \frac{1}{N} \frac{\exp(-2t)}{(1 - \exp(-2t))^2} \exp \left(\frac{\exp(-2t)}{1 - \exp(-2t)} O(d) \right)\\
&\qquad \qquad \qquad \qquad (\text{by the definitions of $\mu(t)$ and $\sigma(t)$})\\
&\leq \frac{1}{N} \frac{1}{t} \exp\left(\frac{O(d)}{t}  \right)
\end{aligned}
\end{equation*}
\end{itemize}
\end{proof}

\begin{lemma}\label{lemma: upper bound of taking expectation w.r.t. t}
$\frac{1}{T - \delta} \int_{\delta}^T \frac{1}{N} \frac{1}{t} \exp \left( \frac{O(d)}{t} \right) dt \lesssim \frac{1}{N} \exp\left( \frac{O(d)}{\delta} \right)$.
\end{lemma}
\begin{proof}
\begin{equation*}
\begin{aligned}
\frac{1}{T - \delta} \int_{\delta}^T \frac{1}{N} \frac{1}{t} \exp \left( \frac{O(d)}{t} \right) dt &= \frac{1}{N} \frac{1}{T - \delta} \int_{1/T}^{1/\delta}  \frac{\exp\left( O(d) s\right)}{s}  ds\\
&\leq \frac{1}{N} \frac{T}{T - \delta} \int_{1/T}^{1/\delta} \exp \left(O(d)s \right)ds \\
&\leq \frac{1}{N} \frac{1}{O(d)} \exp\left( \frac{O(d)}{\delta} \right) \lesssim \frac{1}{N} \exp\left( \frac{O(d)}{\delta} \right). 
\end{aligned}
\end{equation*}
\end{proof}

\begin{lemma}\label{lemma: variance of mean type result}
Suppose $\{y_i\}_{i=1}^N$ are i.i.d samples drawn from the distribution $\ptarget$.
For $v_t(x), p_t(x)$ and $v_t^N(x), p_t^N(x)$ defined in \eqref{eqn: true score second form} and \eqref{eqn: empirical optimal score second form} respectively, we have
\begin{equation*}
\E_{\{y_i\} \sim \ptarget^{\otimes N}} \left[\left\| v_t^N(x) - v_t(x)\right\|^2 \right] \leq \frac{1}{N} \E_{y \sim \ptarget} \left[ \left\| y p_t(x|y) \right\|^2 \right]
\end{equation*}
and
\begin{equation*}
\E_{\{y_i\} \sim \ptarget^{\otimes N}} \left[\left\| p_t^N(x) - p_t(x)\right\|^2 \right] \leq \frac{1}{N} \E_{y \sim \ptarget} \left[ p_t(x|y)^2 \right].
\end{equation*}
\end{lemma}
\begin{remark}
Define $f_{t,x}(y) := y p_t(x|y)$. Due to the randomness in $y$, $f_{t,x}$ is also a random variable. According to the definition~\eqref{eqn: true score second form}, $v_t(x)=\int y \pcond_t(x|y)\ptarget(y)dy=\E_{\ptarget}[f_{t,x}(y)]$ is the mean of random variable $f_{t,x}(y)$, and $v_t^N(x)=\frac{1}{N}\sum_if_{t,x}(y_i)$ is the ensemble average of $N$ realizations of $f_{t,x}$. It is always true that the variance of the ensemble average is $\frac{1}{N}$ of the variance of the original random variable, so naturally:
\[
\E_{\{y_i\} \sim \ptarget^{\otimes N}} \left[\left\| v_t^N(x) - v_t(x)\right\|^2 \right] = \frac{1}{N}\mathrm{Var}_{\ptarget}[f_{t,x}(y)]\leq \frac{1}{N}\E_{\ptarget}\|f_{t,x}\|^2\,.
\]
\end{remark}
\begin{proof}
We denote $f_{t,x}(y) := y \pcond_t(x|y)$.
By the definitions of $v_t^N(x)$ and $v_t(x)$, we can compute
\begin{equation*}
\begin{aligned}
\E_{\{y_i\} \sim \ptarget^{\otimes N}} \left[\left\| v_t^N(x) - v_t(x) \right\|^2 \right] &= \E_{\{y_i\} \sim \ptarget^{\otimes N}} \left[\left\| \frac{1}{N} \sum_{i=1}^N \left(f_{t,x}(y_i) - \E_{y \sim \ptarget}[f_{t,x}(y)] \right) \right\|^2 \right]\\
&= \frac{1}{N} \E_{y \sim \ptarget} \left[\left\| f_{t,x}(y) - \E_{y \sim \ptarget} [f_{t,x}(y)] \right\|^2 \right]\\
&\leq \frac{1}{N} \E_{y \sim \ptarget} \left[ \left\| f_{t,x}(y) \right\|^2 \right]
= \frac{1}{N} \E_{y \sim \ptarget} \left[ \left\| y \pcond_t(x|y) \right\|^2 \right]\,.
\end{aligned}
\end{equation*}
With similar computations, one can show
\begin{equation*}
\E_{\{y_i\} \sim \ptarget^{\otimes N}} \left[\left\| \pf_t^N(x) - \pf_t(x)\right\|^2 \right] \leq \frac{1}{N} \E_{y \sim \ptarget} \left[ \pcond_t(x|y)^2 \right].
\end{equation*}
\end{proof}

\begin{lemma}\label{lemma: bound for empirical division}
Given a collection of vectors $\{y_i\}_{i=1}^N$, for any fixed $x\in \R^d$ and $t \in [0,T]$, the following inequality holds
\begin{equation*}
\frac{\left\|v_t^N(x) \right\|^2}{ p_t^N(x)^2} = \left\| \sum_{i=1}^N \frac{p_t(x|y_i)}{\sum_{j=1}^N p_t(x|y_j)} y_i\right\|^2 \lesssim \frac{\sigma(t)^2}{\mu(t)^2} \|x\|^2 + \frac{1}{\mu(t)^2} \frac{1}{N} \sum_{i=1}^N \left\| x - \mu(t)y_i  \right\|^2,
\end{equation*}
where $v_t^N$ and $p_t^N$ are defined in \eqref{eqn: empirical optimal score second form}, $p_t(x|y)$ is the Green's function defined in \eqref{eqn:green} and $\mu(t) = e^{-t}, \sigma(t)^2 = 1 - e^{-2t}$ as defined in \eqref{eqn: mu and sigma}.
If we further assume that $\|y_i\|_2^2 \leq R^2$ for all $i \in [N]$, then we have
\begin{equation*}
\frac{\left\|v_t^N(x) \right\|^2}{ p_t^N(x)^2} = \left\| \sum_{i=1}^N \frac{p_t(x|y_i)}{\sum_{j=1}^N p_t(x|y_j)} y_i\right\|^2 \leq R^2\, .
\end{equation*}
\end{lemma}
\begin{proof}
We can compute that
\begin{equation*}
\begin{aligned}
\frac{\left\|v_t^N(x) \right\|^2}{ p_t^N(x)^2} &= \left\| \sum_{i=1}^N \frac{p_t(x|y_i)}{\sum_{j=1}^N p_t(x|y_j)} y_i \right\|^2 \\
&= \left\| \sum_{i=1}^N \frac{\exp \left( - \frac{\|x - \mu(t) y_i\|^2}{2\sigma(t)^2} \right)}{\sum_{j=1}^N \exp \left(- \frac{\|x - \mu(t) y_j\|^2}{2\sigma(t)^2} \right)} \frac{\sigma(t)}{\mu(t)} \frac{\left( \mu(t) y_i - x + x \right)}{\sigma(t)} \right\|^2\\
&\lesssim \frac{\sigma(t)^2}{\mu(t)^2} \left( \|x\|^2  + \left\| \sum_{i=1}^N \frac{\exp \left( - \frac{\|x - \mu(t) y_i\|^2}{2\sigma(t)^2} \right)}{\sum_{j=1}^N \exp \left(- \frac{\|x - \mu(t) y_j\|^2}{2\sigma(t)^2} \right)} \frac{\left( \mu(t)y_i - x \right)}{\sigma(t)} \right\|^2 \right)\\
&\leq \frac{\sigma(t)^2}{\mu(t)^2} \|x\|^2 + \frac{1}{\mu(t)^2} \frac{1}{N} \sum_{i=1}^N \left\| x - \mu(t)y_i  \right\|^2 \qquad (\text{by Lemma \ref{lemma: weight avg bounded by uniform avg}})
\end{aligned}
\end{equation*}
If we assume that $\|y_i\|_2^2 \leq R^2$ for $i \in [N]$, then we have
\begin{equation*}
\frac{\left\|v_t^N(x) \right\|^2}{ p_t^N(x)^2} = \left\| \sum_{i=1}^N \frac{p_t(x|y_i)}{\sum_{j=1}^N p_t(x|y_j)} y_i\right\|^2 \leq R^2 \left\| \frac{\sum_{i=1}^N p_t(x|y_i)}{\sum_{j=1}^N p_t(x|y_j)} \right\|^2 = R^2 \, .
\end{equation*}
\end{proof}

\begin{lemma}\label{lemma: upper bound w.r.t y gaussian case}
Under the same assumptions as in Theorem \ref{theorem: approximation error new version} Case 1, for fixed $t \in [\delta, T]$ and $x \in \R^d$, we have
\begin{equation*}
\begin{aligned}
\E_{\{y_i\} \sim \ptarget^{\otimes N}} &\left[\frac{1}{p_t(x)^2} \left\| v_t^N(x) - v_t(x) \right\|^2 + \left( \frac{\pf_t^N(x) - \pf_t(x)}{\pf_t(x)} \right)^2 \frac{\left\| v_t^N(x) \right\|^2}{\pf_t^N(x)^2} \right]\\
&\qquad \qquad \qquad \qquad \qquad \quad \;\; \leq \frac{\|x\|^2 + \mathfrak{m}_2^2}{N \mu(t)^2} \exp \left( \frac{\|x - \mu(t) \mu_{\ptarget}\|^2}{2 \left( \frac{(\sigma(t)^2 + \mu(t)^2 \sigma_{\ptarget}^2)(\sigma(t)^2/2 + \mu(t)^2 \sigma_{\ptarget}^2)}{\mu(t)^2 \sigma_{\ptarget}^2} \right)} \right)
\end{aligned}
\end{equation*}
Here we denote $\mathfrak{m}_2^2 := \E_{y \sim \ptarget} \left[\|y\|^2 \right] = \|\mu_{\ptarget}\|^2 + d \sigma_{\ptarget}^2$.
\end{lemma}

\begin{proof}
Denote
\begin{equation*}
A_1 := \E_{\{y_i\} \sim \ptarget^{\otimes N}} \left[\frac{1}{p_t(x)^2} \left\| v_t^N(x) - v_t(x) \right\|^2 \right] \qquad \text{and} \qquad A_2 := \E_{\{y_i\} \sim \ptarget^{\otimes N}} \left[\left( \frac{\pf_t^N(x) - \pf_t(x)}{\pf_t(x)} \right)^2 \frac{\left\| v_t^N(x) \right\|^2}{\pf_t^N(x)^2} \right]
\end{equation*}
We now bound terms $A_1$ and $A_2$ respectively.
For term $A_1$, we have
\begin{equation}\label{eqn: upper bound for A1 gaussian case}
\begin{aligned}
A_1 &= \frac{1}{p_t(x)^2} \E_{\{y_i\} \sim \ptarget^{\otimes N}} \left[ \left\| v_t^N(x) - v_t(x) \right\|^2 \right]\\
&\leq \frac{1}{N} \frac{1}{p_t(x)^2} \E_{y \sim \ptarget} \left[\left\| y p_t(x|y) \right\|^2 \right] \qquad (\text{by Lemma \ref{lemma: variance of mean type result}})\\
&\lesssim \frac{\|x\|^2 + \mathfrak{m}_2^2}{N} \exp \left( \frac{\|x - \mu(t) \mu_{\ptarget}\|^2}{2\left( \frac{\sigma(t)^2 + \mu(t)^2 \sigma_{\ptarget}^2}{2} \right)}\right) \exp \left( -\frac{\|x - \mu(t) \mu_{\ptarget}\|^2}{2 \left( \sigma(t)^2 / 2 + \mu(t)^2 \sigma_{\ptarget}^2 \right)} \right) \qquad (\text{by Lemma \ref{lemma: useful quantities}})\\
&= \frac{\|x\|^2 + \mathfrak{m}_2^2}{N} \exp \left( \frac{\|x - \mu(t) \mu_{\ptarget}\|^2}{2 \left( \frac{(\sigma(t)^2 + \mu(t)^2 \sigma_{\ptarget}^2)(\sigma(t)^2/2 + \mu(t)^2 \sigma_{\ptarget}^2)}{\mu(t)^2 \sigma_{\ptarget}^2} \right)} \right)
\end{aligned}
\end{equation}
For term $A_2$, by Lemma \eqref{lemma: bound for empirical division} we obtain
\begin{equation*}
\begin{aligned}
A_2 &= \E_{\{y_i\} \sim \ptarget^{\otimes N}} \left[\left( \frac{\pf_t^N(x) - \pf_t(x)}{\pf_t(x)} \right)^2 \frac{\left\| v_t^N(x) \right\|^2}{\pf_t^N(x)^2} \right]\\
\lesssim &\frac{1}{\mu(t)^2}  \frac{1}{p_t(x)^2} \left( \sigma(t)^2 \|x\|^2 \E_{\{y_i\} \sim \ptarget^{\otimes N}} \left[ \left( p_t^N(x) - p_t(x) \right)^2 \right] +  \E_{\{y_i\} \sim \ptarget^{\otimes N}} \left[ \frac{1}{N}\sum_{i=1}^N \left\| x - \mu(t)y_i  \right\|^2 \left( p_t^N(x) - p_t(x) \right)^2 \right] \right)\\
:= &\frac{1}{\mu(t)^2}  \frac{1}{p_t(x)^2} \left(A_{2,1} + A_{2,2} \right)
\end{aligned}
\end{equation*}
By Lemma \ref{lemma: useful quantities}, we have
\begin{equation*}
    A_{2,1} = \sigma(t)^2 \|x\|^2 \E_{\{y_i\} \sim \ptarget^{\otimes N}} \left[ \left( p_t^N(x) - p_t(x) \right)^2 \right] \lesssim \frac{\sigma(t)^2 \|x\|^2}{N} \exp \left( -\frac{\|x - \mu(t) \mu_{\ptarget}\|^2}{2 \left( \sigma(t)^2 / 2 + \mu(t)^2 \sigma_{\ptarget}^2 \right)} \right)
\end{equation*}
By Lemma \ref{lemma: upper bound of A22}, we know that
\begin{equation*}
A_{2,2} = \E_{\{y_i\} \sim \ptarget^{\otimes N}} \left[ \frac{1}{N}\sum_{i=1}^N \left\| x - \mu(t)y_i  \right\|^2 \left( p_t^N(x) - p_t(x) \right)^2 \right] \lesssim \frac{\|x\|^2 + \mathfrak{m}_2^2 }{N}  \exp\left(- \frac{\left\|x - \mu(t)\mu_{\ptarget} \right\|^2}{2 \left(\sigma(t)^2/2 + \mu(t)^2 \sigma_{\ptarget}^2 \right)} \right)
\end{equation*}
Then we can obtain the upper bound for term $A_2$, i.e.
\begin{equation}\label{eqn: upper bound for A2 gaussian case}
\begin{aligned}
A_2 &\lesssim \frac{1}{\mu(t)^2}  \frac{1}{p_t(x)^2} \left(A_{2,1} + A_{2,2} \right) \\
&\lesssim \frac{1}{\mu(t)^2}  \frac{1}{p_t(x)^2} \frac{\|x\|^2 + \mathfrak{m}_2^2 }{N}  \exp\left(- \frac{\left\|x - \mu(t)\mu_{\ptarget} \right\|^2}{2 \left(\sigma(t)^2/2 + \mu(t)^2 \sigma_{\ptarget}^2 \right)} \right)\\
&\lesssim \frac{1}{\mu(t)^2} \frac{\|x\|^2 + \mathfrak{m}_2^2 }{N} \exp \left( \frac{\|x - \mu(t) \mu_{\ptarget}\|^2}{2 \left( \frac{(\sigma(t)^2 + \mu(t)^2 \sigma_{\ptarget}^2)(\sigma(t)^2/2 + \mu(t)^2 \sigma_{\ptarget}^2)}{\mu(t)^2 \sigma_{\ptarget}^2} \right)} \right)  
\end{aligned}
\end{equation}
We finish the proof by combining the upper bounds of terms $A_1$ and $A_2$ derived in \eqref{eqn: upper bound for A1 gaussian case} and \eqref{eqn: upper bound for A2 gaussian case}.
\end{proof}

\begin{lemma}\label{lemma: upper bound of A22}
Under the same assumptions as in Lemma \ref{lemma: upper bound w.r.t y gaussian case}, we have
\begin{equation*}
    A_{2,2} := \E_{\{y_i\} \sim \ptarget^{\otimes N}} \left[ \frac{1}{N}\sum_{i=1}^N \left\| x - \mu(t)y_i  \right\|^2 \left( p_t^N(x) - p_t(x) \right)^2 \right] \lesssim \frac{\|x\|^2 + \mathfrak{m}_2^2 }{N}  \exp\left(- \frac{\left\|x - \mu(t)\mu_{\ptarget} \right\|^2}{2 \left(\sigma(t)^2/2 + \mu(t)^2 \sigma_{\ptarget}^2 \right)} \right)
\end{equation*}
\end{lemma}

\begin{proof}
For notation simplicity, we denote $g_{t,x}(y) := p_t(x|y)$ and use $\E_{\{y_i\}}$ as a short notation of $\E_{\{y_i\} \sim \ptarget^{\otimes N}}$ when the context is clear.
Then we have
\begin{equation*}
\begin{aligned}
A_{2,2} &=\E_{\{y_i\} \sim \ptarget^{\otimes N}} \left[ \frac{1}{N}\sum_{i=1}^N \left\| x - \mu(t)y_i  \right\|^2 \left( p_t^N(x) - p_t(x) \right)^2 \right]\\
&= \frac{1}{N} \sum_{i=1}^N \E_{\{y_i\}} \left[\left\|x - \mu(t) y_i \right\|^2 \left( \frac{1}{N} \sum_{j=1}^N \left(g_{t,x}(y_j) - \E_{y_j}[g_{t,x}(y_j)] \right) \right)^2 \right]
\end{aligned}
\end{equation*}
For every $i \in [N]$, we can compute
\begin{equation*}
\begin{aligned}
&\E_{\{y_k\}_{k=1}^N} \left[ \left\| x - \mu(t)y_i \right\|^2 \left(\frac{1}{N} \sum_{j=1}^N \left(g_{t,x}(y_j) - \E_{y_j} [g_{t,x}(y_j)] \right) \right)^2 \right]\\ &= \E_{y_i}\left[ \left\| x - \mu(t)y_i \right\|^2 \E_{\{y_j\}_{j\neq i}^N} \left[ \frac{1}{N^2} \left( \left( g_{t,x}(y_i) - \E_{y_i} [g_{t,x}(y_i)] \right) + \sum_{j\neq i}^N \left( g_{t,x}(y_j) - \E_{y_j} [g_{t,x}(y_j)] \right) \right)^2 \right] \right]\\
&\lesssim \frac{1}{N^2}  \E_{y_i}\left[ \left\| x - \mu(t)y_i \right\|^2 \E_{\{y_j\}_{j\neq i}^N} \left[  \left( g_{t,x}(y_i) - \E_{y_i} [g_{t,x}(y_i)] \right)^2 + \left( \sum_{j\neq i}^N \left( g_{t,x}(y_j) - \E_{y_j} [g_{t,x}(y_j)] \right) \right)^2 \right] \right]\\
&= \frac{1}{N^2} \Bigg( \E_{y_i} \left[ \left\| x - \mu(t)y_i \right\|^2 \left( g_{t,x}(y_i) - \E_{y_i} [g_{t,x}(y_i)] \right)^2 \right]\\
&\qquad \qquad + \E_{y_i} \left[\left\| x - \mu(t)y_i \right\|^2  \right] \E_{\{y_j\}_{j\neq i}^N} \left[ \left( \sum_{j\neq i}^N \left( g_{t,x}(y_j) - \E_{y_j} [g_{t,x}(y_j)] \right) \right)^2 \right] \Bigg)\\
&\leq \frac{1}{N^2} \left( \E_{y_i} \left[ \left\| x - \mu(t)y_i \right\|^2 \left( g_{t,x}(y_i) - \E_{y_i} [g_{t,x}(y_i)] \right)^2 \right] + \E_{y_i} \left[\left\| x - \mu(t)y_i \right\|^2  \right] (N-1) \E_{y} \left[ \left( g_{t,x}(y) \right)^2 \right] \right)\\
\end{aligned}
\end{equation*}
Therefore, we have
\begin{equation*}
\begin{aligned}
A_{2,2} &= \frac{1}{N} \sum_{i=1}^N \E_{\{y_i\}} \left[ \left\| x - \mu(t)y_i \right\|^2 \left(\frac{1}{N} \sum_{j=1}^N \left(g_{t,x}(y_j) - \E_{y_j} [g_{t,x}(y_j)] \right) \right)^2 \right]\\
&\lesssim \frac{1}{N} \sum_{i=1}^N \frac{1}{N^2} \bigg( \E_{y_i } \left[ \left\| x - \mu(t)y_i \right\|^2 \left( g_{t,x}(y_i) - \E_{y_i } [g_{t,x}(y_i)] \right)^2 \right]\\
&\qquad \qquad \qquad \;\; + (N-1)\E_{y_i} \left[\left\| x - \mu(t)y_i \right\|^2  \right]   \E_{y} \left[ \left( g_{t,x}(y) \right)^2 \right] \bigg)\\
&= \frac{1}{N^2} \E_{y} \left[ \left\| x - \mu(t)y \right\|^2 \left( g_{t,x}(y) - \E_{y} [g_{t,x}(y)] \right)^2 \right] + \frac{N-1}{N^2} \E_{y} \left[\left\| x - \mu(t)y \right\|^2  \right] \E_{y } \left[ \left( g_{t,x}(y) \right)^2 \right] \\
&:= A_{2,2,1} + A_{2,2,2}
\end{aligned}
\end{equation*}
Note that
\begin{equation*}
\begin{aligned}
A_{2,2,1} &= \frac{1}{N^2} \E_{y} \left[ \left\| x - \mu(t)y \right\|^2 \left( g_{t,x}(y) - \E_{y} [g_{t,x}(y)] \right)^2 \right]\\
&\lesssim \frac{1}{N^2} \left(\E_y \left[ \left\|x - \mu(t)y \right\|^2 g_{t,x}(y)^2 \right]  + \left(\E_y \left[ g_{t,x}(y)\right] \right)^2 \E_y \left[ \|x - \mu(t)y\|^2 \right] \right)\\
&:= \frac{1}{N^2} \left( B_1 + B_2\right)
\end{aligned}
\end{equation*}
For term $B_1$, we have
\begin{equation*}
\begin{aligned}
B_1 &= \E_y \left[ \left\|x - \mu(t)y \right\|^2 g_{t,x}(y)^2 \right]\\
&= \E_y \left[ \|x - \mu(t)y\|^2 \exp\left( - \frac{\|x -\mu(t)y\|^2}{2 \left( \sigma(t)^2 / 2 \right)} \right) \right]\\
&= \E_{\widetilde{y}} \left[ \|\widetilde{y}\|^2 \exp \left( -\frac{\left\|\widetilde{y} \right\|^2}{2 \left( \sigma(t)^2 / 2 \right)} \right) \right], \qquad \text{where}\; \widetilde{y} := x - \mu(t)y \sim \mathcal{N}(\widetilde{y}; x - \mu(t)\mu_{\ptarget}, \mu(t)^2 \sigma_{\ptarget}^2 I_{d\times d})\\
&\lesssim \int \left\| \widetilde{y} \right\|^2 \exp \left( -\frac{\left\|\widetilde{y} \right\|^2}{2 \left( \sigma(t)^2 / 2 \right)} \right) \exp\left(- \frac{\left\| \widetilde{y} - (x - \mu(t)\mu_{\ptarget}) \right\|^2}{2\mu(t)^2 \sigma_{\ptarget}^2} \right) dy\\
&\lesssim \mathcal{N}\left(x; \mu(t)\mu_{\ptarget}, \left( \frac{\sigma(t)^2}{2} + \mu(t)^2 \sigma_{\ptarget}^2 \right) I_{d\times d} \right) \E_{\widehat{Y}} \left[ \left\| \widehat{Y} \right\|^2 \right], \qquad (\text{similar to the computations in \eqref{eqn: comp conv of two Gaussians}})\\
&\qquad \qquad \text{where} \; \widehat{Y} \sim \mathcal{N} \left(\widehat{y}; \frac{\sigma(t)^2 (x - \mu(t)\mu_{\ptarget})}{\sigma(t)^2 + 2\mu(t)^2 \sigma_{\ptarget}^2 }, \frac{\sigma(t)^2 \mu(t)^2 \sigma_{\ptarget}^2}{\sigma(t)^2 + 2\mu(t)^2 \sigma_{\ptarget}^2} I_{d\times d}  \right)\\
&\lesssim  \left( \left\| x - \mu(t)\mu_{\ptarget} \right\|^2 + d \sigma_{\ptarget}^2 \right) \exp\left(- \frac{\left\|x - \mu(t)\mu_{\ptarget} \right\|^2}{2 \left(\sigma(t)^2/2 + \mu(t)^2 \sigma_{\ptarget}^2 \right)} \right)\\
&\lesssim \left( \|x\|^2 + \|\mu_{\ptarget}\|^2 + d\sigma_{\ptarget}^2 \right) \exp\left(- \frac{\left\|x - \mu(t)\mu_{\ptarget} \right\|^2}{2 \left(\sigma(t)^2/2 + \mu(t)^2 \sigma_{\ptarget}^2 \right)} \right) \\
&= \left( \|x\|^2 + \mathfrak{m}_2^2 \right) \exp\left(- \frac{\left\|x - \mu(t)\mu_{\ptarget} \right\|^2}{2 \left(\sigma(t)^2/2 + \mu(t)^2 \sigma_{\ptarget}^2 \right)} \right)
\end{aligned}
\end{equation*}
For term $B_2$, we have
\begin{equation*}
\left( \E_y \left[ g_{t,x}(y) \right] \right)^2 = \left( \E_{y} \left[ \exp \left( -\frac{\|x - \mu(t) y\|^2}{2 \sigma(t)^2} \right)\right] \right)^2 \lesssim \exp \left( -\frac{\|x - \mu(t) \mu_{\ptarget}\|^2}{2\left( \frac{\sigma(t)^2 + \mu(t)^2 \sigma_{\ptarget}^2}{2} \right)}\right) \qquad (\text{by Lemma \ref{lemma: useful quantities}})
\end{equation*}
and 
\begin{equation*}
\begin{aligned}
\E_y \left[ \left\| x - \mu(t)y \right\|^2 \right] &= \E_{\widetilde{y}} \left[ \left\| \widetilde{y} \right\|^2 \right], \qquad \text{where}\; \widetilde{y} := x - \mu(t)y \sim \mathcal{N}(\widetilde{y}; x - \mu(t)\mu_{\ptarget}, \mu(t)^2 \sigma_{\ptarget}^2 I_{d\times d})\\
&= \left\| x - \mu(t)\mu_{\ptarget} \right\|^2 + d \mu(t)^2 \sigma_{\ptarget}^2\\
&\lesssim \|x\|^2 + \mathfrak{m}_2^2
\end{aligned}
\end{equation*}
Therefore, we know that
\begin{equation*}
\begin{aligned}
B_2 &= \left(\E_y \left[ g_{t,x}(y)\right] \right)^2 \E_y \left[ \|x - \mu(t)y\|^2 \right] \lesssim \left( \|x\|^2 + \mathfrak{m}_2^2 \right) \exp \left( -\frac{\|x - \mu(t) \mu_q\|^2}{2\left( \frac{\sigma(t)^2 + \mu(t)^2 \sigma_{\ptarget}^2}{2} \right)}\right)\\
&\leq \left( \|x\|^2 + \mathfrak{m}_2^2 \right) \exp\left(- \frac{\left\|x - \mu(t)\mu_{\ptarget} \right\|^2}{2 \left(\sigma(t)^2/2 + \mu(t)^2 \sigma_{\ptarget}^2 \right)} \right) 
\end{aligned}
\end{equation*}
Then, we can have the upper bound for $A_{2,2,1}$, i.e.
\begin{equation*}
A_{2,2,1} = \frac{1}{N^2} \left(B_1 + B_2 \right) \lesssim \frac{1}{N^2} \left( \|x\|^2 + \mathfrak{m}_2^2 \right) \exp\left(- \frac{\left\|x - \mu(t)\mu_{\ptarget} \right\|^2}{2 \left(\sigma(t)^2/2 + \mu(t)^2 \sigma_{\ptarget}^2 \right)} \right) 
\end{equation*}
Similar to the computations for term $A_{2,2,1}$, we can compute the upper bound of $A_{2,2,2}$ as the following:
\begin{equation*}
\begin{aligned}
A_{2,2,2} = \frac{N-1}{N^2} \E_{y} \left[\left\| x - \mu(t)y \right\|^2  \right] \E_{y} \left[ \left( g_{t,x}(y) \right)^2 \right] \lesssim \frac{N-1}{N^2} \left( \|x\|^2 + \mathfrak{m}_2^2 \right) \exp\left(- \frac{\left\|x - \mu(t)\mu_{\ptarget} \right\|^2}{2 \left(\sigma(t)^2/2 + \mu(t)^2 \sigma_{\ptarget}^2 \right)} \right)
\end{aligned}
\end{equation*}
Therefore, for term $A_{2,2}$, we have
\begin{equation*}
A_{2,2} \lesssim A_{2,2,1} + A_{2,2,2} \lesssim \frac{\|x\|^2 + \mathfrak{m}_2^2}{N}  \exp\left(- \frac{\left\|x - \mu(t)\mu_{\ptarget} \right\|^2}{2 \left(\sigma(t)^2/2 + \mu(t)^2 \sigma_{\ptarget}^2 \right)} \right)
\end{equation*}
\end{proof}

\begin{lemma}[Convolution of two Gaussian distributions]\label{lemma: conv of two Gaussians}
Let $f_X(x) = \mathcal{N}(x; \mu_X, \sigma_X^2 I_{d\times d})$ and $f_Y(y) = \mathcal{N}(y; \mu_Y, \sigma_Y^2 I_{d\times d})$, then
\begin{equation*}
f_Z(z) := \int f_X(z-y) f_Y(y)dy = \mathcal{N}(z; \mu_X + \mu_Y, (\sigma_X^2 + \sigma_Y^2) I_{d\times d})
\end{equation*}
\end{lemma}

\begin{proof}
One can compute that
\begin{equation*}
\begin{aligned}
f_Z(z) &= \int f_X(z-y) f_Y(y)dy \\
&\propto \int \exp\left( - \frac{\|z - y - \mu_x\|^2}{2\sigma_X^2} \right) \exp\left( -\frac{\|y - \mu_Y\|^2}{2\sigma_Y^2} \right) dy\\
&= \int \exp\bigg( -\frac{1}{2\sigma_X^2 \sigma_Y^2} \bigg[ \sigma_Y^2 \left(\|z\|^2 + \|y\|^2 + \|\mu_X\|^2 - 2z^T y -2 z^T\mu_X + 2\mu_X^T y \right)\\
&\qquad \qquad \qquad \qquad \qquad  + \sigma_X^2 \left( \|y\|^2 - 2\mu_Y^T y + \|\mu_Y\|^2 \right)\bigg] \bigg) dy\\
&\propto \int \exp \left(-\frac{1}{2\sigma_X^2 \sigma_Y^2} \left[(\sigma_X^2 + \sigma_Y^2)\|y\|^2 - 2\left( \sigma_Y^2 (z - \mu_X) + \sigma_X^2 \mu_Y \right)^T y + \sigma_Y^2 \|z\|^2 \right]  \right)
\end{aligned}
\end{equation*}
Define $\sigma_Z := \sqrt{\sigma_X^2 + \sigma_Y^2}$, and completing the square:
\begin{equation}\label{eqn: comp conv of two Gaussians}
\begin{aligned}
f_Z(z) &\propto \exp\left(-\frac{\|z\|^2}{2\sigma_X^2} \right) \int \exp \left( -\frac{1}{2\left(\frac{\sigma_X\sigma_Y}{\sigma_Z} \right)^2} \left( \|y\|^2 - \frac{2}{\sigma_Z^2} (\sigma_Y^2 (z - \mu_X) + \sigma_X^2 \mu_Y)^T y  \right) \right) dy\\
&\propto \exp\left( -\frac{\|z\|^2}{2\sigma_X^2} + \frac{\|\sigma_Y^2(z- \mu_X) + \sigma_X^2 \mu_Y\|^2}{2\sigma_Z^2 (\sigma_X \sigma_Y)^2} \right) \int \exp\left( -\frac{1}{2\left(\frac{\sigma_X\sigma_Y}{\sigma_Z} \right)^2} \left\| y - \frac{\sigma_Y^2(z - \mu_X) + \sigma_X^2 \mu_Y}{\sigma_Z^2}\right\|^2 \right) dy \\
&\propto \exp\left( -\frac{\|z - (\mu_X + \mu_Y)\|^2}{2(\sigma_X^2 + \sigma_Y^2)} \right) E_{\widehat{Y}}[\mathbb{I}\{\widehat{Y} \leq +\infty\}],\; \text{where} \; \widehat{Y} \sim \mathcal{N}\left( \widehat{y}; \frac{\sigma_Y^2(z - \mu_X) + \sigma_X^2 \mu_Y}{\sigma_Z^2}, \frac{\sigma_X^2 \sigma_Y^2}{\sigma_Z^2} I_{d\times d} \right)\\
&\propto \mathcal{N}\left(z; \mu_X + \mu_Y, (\sigma_X^2 + \sigma_Y^2 \right) I_{d\times d})
\end{aligned}
\end{equation}
\end{proof}

\begin{lemma}\label{lemma: useful quantities}
Suppose $y \sim \ptarget = \mathcal{N} \left(y; \mu_{\ptarget}, \sigma_{\ptarget}^2 I_{d \times d} \right)$, then one can compute the following quantities:
\begin{itemize}
    \item[1.] 
    \begin{equation*}
    p_t(x) = \E_{y \sim \ptarget} [p_t(x|y)] = \mathcal{N} \left(x;\mu(t)\mu_{\ptarget}, \left(\sigma(t)^2 + \mu(t)^2 \sigma_{\ptarget}^2 \right) I_{d\times d} \right) := h(x)
    \end{equation*}

    \item[2.]
    \begin{equation*}
    \E_{y \sim \ptarget} \left[ y \exp \left( -\frac{\|x - \mu(t) y\|^2}{2\sigma(t)^2} \right) \right] \propto \left(\frac{\mu(t) \sigma_{\ptarget}^2 x + \sigma(t)^2 \mu_{\ptarget}}{\sigma(t)^2 + \mu(t)^2 \sigma_{\ptarget}^2} \right) h(x)
    \end{equation*}

    \item[3.] 
    \begin{equation*}
    \E_{y \sim \ptarget} \left[ \|y\|^2 \exp \left( -\frac{\|x - \mu(t) y\|^2}{2\sigma(t)^2} \right) \right] \lesssim \left( \|x\|^2 + \mathfrak{m}_2^2  \right) h(x),
    \end{equation*}
    where $\mathfrak{m}_2^2 := \|\mu_{\ptarget}\|^2 + d \sigma_{\ptarget}^2$. Both $\propto$ and $\lesssim$ indicate ignoring the constants.
\end{itemize}
\end{lemma}

\begin{proof}
\begin{itemize}
    \item[1.] 
    \begin{equation}\label{eqn: comp 0 mom}
    \begin{aligned}
    \E_{y \sim \ptarget} [p_t(x|y)] &\propto \E_{y\sim \ptarget} \left[\exp \left( -\frac{\|x - \mu(t) y\|^2}{2 \sigma(t)^2} \right) \right]\\
    &\propto \int \exp\left(-\frac{\|x - \mu(t) y\|^2}{2 \sigma(t)^2}  \right) \exp \left( -\frac{\|y - \mu_{\ptarget}\|^2}{2\sigma_{\ptarget}^2} \right)dy\\
    &= \int \exp\left(- \frac{\left\| x/\mu(t) - y \right\|^2}{2 \sigma(t)^2 / \mu(t)^2} \right) \exp \left( -\frac{\|y - \mu_{\ptarget}\|^2}{2\sigma_{\ptarget}^2} \right)dy\\
    &= \mathcal{N}\left(\frac{x}{\mu(t)}; 0, \frac{\sigma(t)^2}{\mu(t)^2} I_{d\times d} \right) * \mathcal{N}\left(y; \mu_{\ptarget}, \sigma_{\ptarget}^2 I_{d\times d} \right) \\
    &= \mathcal{N} \left(x; \mu(t)\mu_{\ptarget}, \left(\sigma(t)^2 + \mu(t)^2 \sigma_{\ptarget}^2 \right) I_{d\times d} \right) \qquad (\text{by Lemma \ref{lemma: conv of two Gaussians}})
    \end{aligned}
    \end{equation}

    \item[2.]  Similar to the computations in \eqref{eqn: comp conv of two Gaussians} and \eqref{eqn: comp 0 mom}, one can compute
    \begin{equation*}
    \begin{aligned}
    \E_{y \sim \ptarget} \left[y \exp \left( -\frac{\|x - \mu(t) y\|^2}{2 \sigma(t)^2} \right) \right] &\propto \mathcal{N} \left(x; \mu(t)\mu_{\ptarget}, \left(\sigma(t)^2 + \mu(t)^2 \sigma_{\ptarget}^2 \right) I_{d\times d} \right) \E_{\widehat{Y}} \left[ \widehat{Y} \right],\\
    &\qquad  \left(\text{where} \; \widehat{Y} \sim \mathcal{N}\left(\widehat{y}; \frac{\mu(t) \sigma_{\ptarget}^2 x + \sigma(t)^2 \mu_{\ptarget}}{\sigma(t)^2 + \mu(t)^2 \sigma_{\ptarget}^2}, \frac{\sigma(t)^2 \sigma_{\ptarget}^2}{\sigma(t)^2 + \mu(t)^2 \sigma_{\ptarget}^2} I_{d\times d} \right)\right)\\
    &= \left( \frac{\mu(t) \sigma_{\ptarget}^2 x + \sigma(t)^2 \mu_{\ptarget}}{\sigma(t)^2 + \mu(t)^2 \sigma_{\ptarget}^2} \right) h(x) 
    \end{aligned}
    \end{equation*}

    \item[3.]
    \begin{equation*}
    \begin{aligned}
    \E_{y \sim \ptarget} \left[ \|y\|^2 \exp \left( -\frac{\|x - \mu(t) y\|^2}{2\sigma(t)^2} \right) \right] &\propto \mathcal{N} \left(x;\mu(t)\mu_{\ptarget}, \left(\sigma(t)^2 + \mu(t)^2 \sigma_{\ptarget}^2 \right) I_{d\times d} \right) \E_{\widehat{Y}} \left[ \|\widehat{Y}\|^2 \right],\\
    & \left(\text{where} \; \widehat{Y} \sim \mathcal{N}\left(\widehat{y}; \frac{\mu(t) \sigma_{\ptarget}^2 x + \sigma(t)^2 \mu_{\ptarget}}{\sigma(t)^2 + \mu(t)^2 \sigma_{\ptarget}^2}, \frac{\sigma(t)^2 \sigma_{\ptarget}^2}{\sigma(t)^2 + \mu(t)^2 \sigma_{\ptarget}^2} I_{d\times d} \right)\right)\\
    &= \left( \left\| \frac{\mu(t) \sigma_{\ptarget}^2 x + \sigma(t)^2 \mu_{\ptarget}}{\sigma(t)^2 + \mu(t)^2 \sigma_{\ptarget}^2}\right\|^2  + d \frac{\sigma(t)^2 \sigma_{\ptarget}^2}{\sigma(t)^2 + \mu(t)^2 \sigma_{\ptarget}^2} \right) h(x)\\
    &\lesssim \left( \|x\|^2 + \|\mu_{\ptarget}\|^2 + d \sigma_{\ptarget}^2 \right) h(x) = \left( \|x\|^2 + \mathfrak{m}_2^2 \right) h(x)
    \end{aligned}
    \end{equation*}
\end{itemize}
\end{proof}

\begin{lemma}\label{lemma: weight avg bounded by uniform avg}
Given a collection of d-dimensional vectors $\{y_i\}_{i=1}^N$, the following inequality holds
\begin{equation*}
    \left\| \sum_{i=1}^N \frac{\exp\left( - \|y_i\|^2 \right)}{\sum_{j=1}^N \exp \left(- \|y_j\|^2 \right)} y_i \right\|^2 \leq \frac{1}{N} \sum_{i=1}^N \|y_i\|^2
\end{equation*}
\end{lemma}
\begin{proof}
Denote $w_i^N := \frac{\exp\left( - \|y_i\|^2 \right)}{\sum_{j=1}^N \exp \left(- \|y_j\|^2 \right)}$ for all $i=1,2, \dots, n$, then we can compute
\begin{equation*}
\begin{aligned}
 \left\| \sum_{i=1}^N \frac{\exp\left( - \|y_i\|^2 \right)}{\sum_{j=1}^N \exp \left(- \|y_j\|^2 \right)} y_i \right\|^2 &= \left\| \sum_{i=1}^N w_i^N y_i \right\|^2\\
 &= \sum_{i=1}^N \sum_{j=1}^N w_i^N w_j^N y_i^T y_j\\
 &\leq \frac{1}{2} \sum_{i=1}^N \sum_{j=1}^N w_i^N w_j^N \left( \|y_i\|^2 + \|y_j\|^2 \right)\\
 &= \sum_{i=1}^N w_i^N \|y_i\|^2 \qquad (\text{by the fact that $\sum_{i=1}^N w_i^N = 1$})\\
 &\leq \frac{1}{N} \sum_{i=1}^N \|y_i\|^2. 
\end{aligned}
\end{equation*}
\end{proof}

\begin{lemma}\label{lemma: lower bound of p_t(x)}
For any $\lambda > 0$, with the Green's function $p_t(x|y)$ defined in \eqref{eqn:green}, $p_t(x) := \int p_t(x|y) \ptarget(y) dy$ is lower bounded by
\begin{equation*}
\begin{aligned}
p_t(x) &\geq \frac{1}{\left(2\pi \sigma(t)^2 \right)^{d/2}} \exp\left( -\frac{1 + \lambda \mu(t)}{2\sigma(t)^2} \|x\|^2 \right) \int \exp\left(-\frac{\mu(t) + \lambda \mu(t)^2}{2 \lambda \sigma(t)^2} \|y\|^2 \right) \ptarget(y) dy\\
&:= \frac{1}{\left(2\pi \sigma(t)^2 \right)^{d/2}} \exp\left( -\frac{1 + \lambda \mu(t)}{2\sigma(t)^2} \|x\|^2 \right) K_t
\end{aligned}
\end{equation*}
\end{lemma}
\begin{proof}
It comes from the direct computation:
\begin{equation*}
\begin{aligned}
p_t(x) &= \int p_t(x|y) \ptarget(y)dy\\
&= \int \frac{1}{\left(2\pi \sigma(t)^2 \right)^{d/2}} \exp\left( -\frac{\|x - \mu(t)y\|^2}{2\sigma(t)^2} \right) \ptarget(y)dy\\
&= \frac{1}{\left(2\pi \sigma(t)^2 \right)^{d/2}} \int \exp\left( -\frac{1}{2\sigma(t)^2} \left(\|x\|^2 - 2\mu(t) x^T y + \mu(t)^2 \|y\|^2 \right) \right) \ptarget(y)dy\\
&\geq \frac{1}{\left(2\pi \sigma(t)^2 \right)^{d/2}} \int \exp\left( -\frac{1}{2\sigma(t)^2} \left( \|x\|^2 + \lambda \mu(t) \|x\|^2 + \frac{1}{\lambda} \mu(t) \|y\|^2 + \mu(t)^2 \|y\|^2 \right) \right) \ptarget(y) dy\\
&= \frac{1}{\left(2\pi \sigma(t)^2 \right)^{d/2}} \exp \left( - \frac{1 + \lambda \mu(t)}{2\sigma(t)^2}\|x\|^2 \right) \int \exp\left(-\frac{\mu(t) + \lambda \mu(t)^2}{2 \lambda \sigma(t)^2} \|y\|^2 \right) \ptarget(y) dy,
\end{aligned}
\end{equation*}
where the inequality comes from Young's inequality, i.e. $2a^T b \leq \lambda \|a\|^2 + \frac{1}{\lambda} \|b\|^2$ for any $\lambda > 0$.
\end{proof}

\section{Memorization effects}\label{appendix: memorization effects}
In this section, we provide the proof of Propositions \ref{prop: OU with empirical and kde} and \ref{prop: DDPM with empirical optimal and OU with empirical}, and Theorem \ref{thm: DDPM resemble KDE}.
For the completeness, we state all the propositions and theorem again before the proof.

\begin{proposition}
Suppose the training samples $\{y_i\}_{i=1}^N$ satisfy $\|y_i\|_2 \leq d$, for $\delta \geq 0$, $\TV(\pbN_{T-\delta}, \ptargetNgamma) \leq  \frac{d\sqrt{\delta}}{2}$ with $\gamma=\sigma(\delta)$, where $\sigma(\cdot)$ is defined in~\eqref{eqn: mu and sigma}.
\end{proposition}

\begin{proof}[Proof of Proposition \ref{prop: OU with empirical and kde}]
By the definition of total variation, we have
\begin{equation*}
\begin{aligned}
\TV \left(\pbN_{T-\delta}, \ptargetNgamma \right) &= \frac{1}{2} \int_{\R^d} \left| \frac{1}{N} \sum_{i=1}^N \mathcal{N} \left(x; \mu(\delta) y_i, \sigma(\delta)^2 I_{d\times d}  \right) - \frac{1}{N} \sum_{i=1}^N \mathcal{N} \left(x; y_i, \sigma(\delta)^2 I_{d\times d}  \right) \right|dx\\
&\leq \frac{1}{N} \sum_{i=1}^N \frac{1}{2} \int_{\R^d} \left| \mathcal{N} \left(x; \mu(\delta) y_i, \sigma(\delta)^2 I_{d\times d}  \right) - \mathcal{N} \left(x; y_i, \sigma(\delta)^2 I_{d\times d}  \right) \right|dx\\
&= \frac{1}{N} \sum_{i=1}^N \TV \left(\mathcal{N} \left(x; \mu(\delta) y_i, \sigma(\delta)^2 I_{d\times d}  \right), \mathcal{N} \left(x; y_i, \sigma(\delta)^2 I_{d\times d}  \right)\right)\\
&\leq \frac{1}{N} \sum_{i=1}^N \sqrt{\frac{1}{2} \KL \left( \mathcal{N} \left(x; \mu(\delta) y_i, \sigma(\delta)^2 I_{d\times d}  \right) || \mathcal{N} \left(x; y_i, \sigma(\delta)^2 I_{d\times d}  \right)  \right)}\\
&= \frac{1}{N} \sum_{i=1}^N \|y_i\|_2 \frac{1 - \mu(\delta)}{2\sigma(\delta)} \qquad (\text{by Lemma \ref{lemma: KL divergence between two gaussians}})\\
&\leq \frac{1 - \exp(-\delta)}{2\sqrt{1 - \exp(-2\delta)}} d \qquad (\text{by the definitions of $\mu(\delta)$ and $\sigma(\delta)$, and $\|y_i\|_2 \leq d$})\\
&\leq  \frac{d\sqrt{\delta}}{2} \, .
\end{aligned}
\end{equation*}
\end{proof}

\begin{proposition}
Under the same assumptions are in Proposition \ref{prop: OU with empirical and kde}, on the time interval $t \in [0,T]$, the total variation between the output distribution of SGM algorithm \eqref{eqn: DDPM with empirical optimal score function} with the empirical optimal score function $\pbsN_t$ and the KDE approximation $\pbN_t$ -- is bounded by $\TV\left( \pbsN_{t}, \pbN_{t}\right) \leq \frac{d}{2} \exp(-T)$.
\end{proposition}

\begin{proof}[Proof of Proposition \ref{prop: DDPM with empirical optimal and OU with empirical}]
By the data-processing inequality and Lemma \ref{lemma: convergence of forward OU}, we have
\begin{equation*}
\begin{aligned}
\TV(\pbN_t, \pbsN_t) \leq \TV(\pbN_0, \pbsN_0) = \TV(\pfN_T, \pi^d) \leq \frac{d}{2} \exp(-T) \, .
\end{aligned}
\end{equation*}
\end{proof}

\begin{theorem}[SGM with empirical optimal score function resembles KDE]
Under the same assumptions as Proposition \ref{prop: DDPM with empirical optimal and OU with empirical}, SGM algorithm \eqref{eqn: DDPM with empirical optimal score function} with the  empirical optimal score function $\scoreN$ returns a simple Gaussian convolution with the empirical distribution in the form of~\eqref{eqn:pfNt}, and it presents the following behavior:
\begin{itemize}
   \item (\textbf{with early stopping}) for any $\varepsilon > 0$, set $T = \log \frac{d}{\varepsilon}$ and $\delta = \frac{\varepsilon^2}{d}$, we have 
   \begin{equation*}
   \TV(\pbsN_{T-\delta}, \ptargetNgamma) \leq \varepsilon\,,\quad\text{with}\quad \gamma = \sigma(\delta)\,,
   \end{equation*}
        \item (\textbf{without early stopping}) by taking the limit 
        $T \rightarrow +\infty$ and $\delta = 0$, we have $\pbsN_{\infty} = \ptargetN = \frac{1}{N} \sum_{i=1}^N \delta_{y_i}$.
\end{itemize}
\end{theorem}

\begin{proof}[Proof of Theorem \ref{thm: DDPM resemble KDE}]
\item (\textbf{with early stopping}) For $0 \leq \delta < T$, combining Proposition \ref{prop: OU with empirical and kde} and Proposition \ref{prop: DDPM with empirical optimal and OU with empirical} using triangle inequality, we have
\begin{equation}\label{eqn: diff between generated distribution and KDE}
    \TV \left(\pbsN_{T-\delta}, \ptargetNgamma \right) \leq \TV(\pbN_{T-\delta}, \ptargetNgamma) + \TV \left(\pbN_{T-\delta}, \pbsN_{T-\delta} \right) \leq \frac{d}{2} \left( \sqrt{\delta} + \exp(-T) \right)
\end{equation}
For any $\varepsilon > 0$, by choosing $T = \log \frac{d}{\varepsilon}$ and $\delta = \frac{\varepsilon^2}{d}$, we obtain $\TV(\pbsN_{T - \delta}, \ptargetNgamma) \leq \varepsilon$.

(\textbf{without early stopping})
By taking the limit $T \rightarrow +\infty$ and $\delta = 0$ in inequality \eqref{eqn: diff between generated distribution and KDE}, we have $\TV(\pbsN_{\infty}, \ptargetN) \leq 0$.
This implies that $\pbsN_{\infty}$ equals to the empirical distribution $\ptargetN = \frac{1}{N} \sum_{i=1}^N \delta_{y_i}$.

\end{proof}

\begin{lemma}[KL divergence between two Gaussian distributions]\label{lemma: KL divergence between two gaussians}
Let $p = \mathcal{N}(\mu_p, \Sigma_p)$ and $q = \mathcal{N}(\mu_q, \Sigma_q)$ be two Gaussian distributions on $\R^d$.
Then the KL divergence between $p$ and $q$ is
\begin{equation*}
    \KL(p || q) = \frac{1}{2} \left[\log \frac{|\Sigma_q|}{|\Sigma_p|} - d + (\mu_p - \mu_q)^T \Sigma_q^{-1} (\mu_p - \mu_q) + \text{Tr}\left\{ \Sigma_q^{-1} \Sigma_p \right\} \right]
\end{equation*}
\end{lemma}

\begin{lemma}[Convergence of forward OU process]\label{lemma: convergence of forward OU} 
Denote $\pfN_T$ to be the distribution of forward OU process at time $T$ initializing with the empirical distribution $\ptargetN = \frac{1}{N} \sum_{i=1}^N \delta_{y_i}$, where $\{y_i\}_{i=1}^N$ are i.i.d samples such that $\|y_i\|_2 \leq d$. 
Then for $T \geq 1$,
\begin{equation*}
    \TV \left(\pfN_T, \pi^d \right) \leq \frac{d}{2} \exp(-T) \, .
\end{equation*}
\end{lemma}
\begin{proof}
Since $p_t(x|y) = \mathcal{N} \left(x; \exp(-t)y, \sigma(t)^2 I_{d\times d} \right)$, by Lemma \ref{lemma: KL divergence between two gaussians} we have
\begin{equation*}
        \KL \left(p_t(x|y) || \pi^d \right) = \frac{1}{2} \left[-d\log \sigma(t)^2 - d + d \sigma(t)^2 + \left\| \exp(-t) y \right\|^2 \right]
\end{equation*}
    By the convexity of the KL divergence,
    \begin{equation*}
    \begin{aligned}
    \KL \left(\pfN_T || \pi^d \right) &= \KL \left( \int_{\R^d} p_T(x|y) \ptargetN(y)dy \bigg\| \pi^d \right)\\
    &\leq \int \KL \left( p_T(x|y) \big\| \pi^d \right) \ptargetN(y)dy\\
    &= \frac{1}{2} \left[-d\log \sigma(T)^2 - d + d \sigma(T)^2 + \exp(-2T) \E_{y \sim \ptargetN} \left\| y \right\|^2 \right]\\
    &= \frac{1}{2} \left[-d \log \left(1 - \exp(-2T) \right) - d + d\left(1 - \exp(-2T) \right) + \exp(-2T) \E_{y \sim \ptargetN} \left\| y \right\|^2 \right]\\
    &\leq \frac{1}{2} \exp(-2T) \E_{y \sim \ptargetN} [\left\|y \right\|^2] \qquad (\text{by the fact $\log(1-x) \geq -x$ for $x\geq 0$})\\
    &= \frac{1}{2} \exp(-2T) \frac{1}{N} \sum_{i=1}^N \|y_i\|_2^2 \leq \frac{d^2}{2} \exp(-2T)
    \end{aligned}
    \end{equation*}
By the Pinsker's inequality, we have
\begin{equation*}
    \TV \left(\pfN_T, \pi^d \right) \leq \sqrt{\frac{1}{2} \KL \left(\pfN_T || \pi^d \right)} \leq \frac{d}{2} \exp(-T)\, . 
\end{equation*}
\end{proof}

\newpage
\section{Numerical experiments}\label{appendix: numerical experiments}
In this section, we provide the details of the numerical experiments. 
The code is available in \url{https://github.com/SixuLi/DDPM_and_KDE}.\footnote{The implementation of KDE generation is built based on code \url{https://github.com/patrickphat/Generate-Handwritten-Digits-Kernel-Density-Estimation};

The implementation of DDPM on CIFAR10 dataset follows the code \url{https://github.com/sail-sg/DiffMemorize/tree/main}.}

\subsection{Synthetic data distribution}

We consider the target data distribution $\ptarget$ is a $2$-dimensional isotropic Gaussian, i.e. $\ptarget(x) := \mathcal{N}(x; \mu_{\ptarget}, \sigma_{\ptarget}^2 I_{2\times 2})$.
In this case, the law of forward OU process \eqref{eqn: true forward process} $p_t$ and the exact score function $u(t,x)$ defined in \eqref{eqn: general form for score function} have explicit formulations.
Specifically, by Lemma \ref{lemma: useful quantities}, we obtain
\begin{equation}\label{eqn: pt in appendix numerical sec}
p_t(x) = \int p_t(x|y) \ptarget(y) dy = \mathcal{N} \left(x; \mu(t) \mu_{\ptarget}, \left(\sigma(t)^2 + \mu(t)^2 \sigma_{\ptarget}^2  \right) I_{2\times 2} \right)
\end{equation}
\begin{equation}\label{eqn: utx in appendix numerical sec}
    u(t,x) = \frac{\int u(t,x|y) p_t(x|y) \ptarget(y) dy}{\int p_t(x|y) \ptarget(y) dy} = -\frac{1}{\sigma(t)^2} x + \frac{\mu(t)}{\sigma(t)^2} \frac{\int y p_t(x|y)\ptarget(d)dy}{\int p_t(x|y) \ptarget(y) dy} = \frac{\mu(t) \mu_{\ptarget} - x}{\sigma(t)^2 + \mu(t)^2 \sigma_{\ptarget}^2},
\end{equation}
where $\mu(t) = \exp(-t)$ and $\sigma(t)^2 = 1 - \exp(-2t)$ as defined in \eqref{eqn: mu and sigma}.
We set choose $\mu_{\ptarget} = [-5,5]$ and $\sigma_{\ptarget}^2 = 10$ in our experiments.

We first estimate the score approximation error of the empirical optimal score function $s^N_{\{y_i\}}$ across various training sample sizes $N$.
Setting early stopping time $\delta = 0.02$, time interval length $T = 5$, and sample size $N$ ranging from $N=100$ to $N=2000$, we numerically estimate 
\begin{equation}\label{eqn: score approximation error in appdendix numerical sec}
    \E_{\{y_i\} \sim \ptarget^{\otimes N}} \left| \errorNy \right|^2 = \E_{t \sim U[\delta, T]} \E_{\{y_i\} \sim \ptarget^{\otimes N}, x\sim p_t} \left[ \left\|\scoreNy (t,x) - u(t,x) \right\|^2 \right],
\end{equation}
using the empirical average 
\begin{equation}\label{eqn: empirical score approximation error in appendix}
    \widehat{\left| E_{\{y_i\}} \right|}^2 := \frac{1}{K} \frac{1}{M} \sum_{k=1}^K \sum_{m=1}^M \left\| s_{\{y_i\}}^N(t_k, x_m^{t_k}) - u(t_k, x_m^{t_k}) \right\|^2
\end{equation}
This is achieved through repeating the following steps $10$ times and computing the average output:
\begin{itemize}
    \item[1.] Randomly sample $N$ training data $\{y_i\}_{i=1}^N$ from the target distribution $\ptarget$;

    \item[2.] Uniformly sample $\{t_k\}_{k=1}^K$ from time interval $[\delta, T]$ with step size $h = 0.02$ and total number of steps $K = \frac{T}{h}$;

    \item[3.] For each $t_k$, sample $\{x_m^{t_k}\}_{m=1}^M$ (where $M=1000$) from the distribution $p_{t_k}$ as derived in \eqref{eqn: pt in appendix numerical sec};

    \item[4.] Compute the empirical average $\widehat{\left|E_{\{y_i\}} \right|}^2$ \eqref{eqn: empirical score approximation error in appendix} using $\{y_i\}$, $\{t_k\}$ and $\{x_m^{t_k}\}$.
\end{itemize}
The results (shown in Figure \ref{fig: score approximation error}) align with the convergence rate $O(\frac{1}{N})$ as provided in Theorem \ref{thm: approximation error of empirical optimal score function}.

In the second part of our experiments, we generate samples from DDPM using either the exact score function $u(t,x)$ or the empirical optimal score function $s_{\{y_i\}}^N(t,x)$.
We discretize and simulate the SDEs \eqref{eqn: true reverse process} and \eqref{eqn: DDPM with empirical optimal score function} using Euler-maruyama method.
The experiment parameters are: time interval length $T = 5$, discretization step $h = 0.0005$, early stopping times $\delta = 0 \; \text{or} \; 0.01$, and number of training data $N = 100$.
We generate $1000$ new samples from DDPM with $u(t,x)$ and $s^N_{\{y_i\}}(t,x)$ respectively.
Visualization results for $\delta = 0$ and $\delta = 0.01$ are shown in Figure \ref{fig: generated samples 2d-gaussian without early stop} and Figure \ref{fig: generated samples 2d-gaussian with early stopping}.
The samples generated by DDPM with the empirical optimal score function $s^N_{\{y_i\}}(t,x)$ exhibit strong memorization effects, while those from DDPM with the exact score function $u(t,x)$ appear independent of the training data, yet maintain the same distribution.
This numerical observation corroborates our theoretical findings in Theorem \ref{thm: DDPM resemble KDE}.

\begin{figure}[!htb]
    \centering
    \begin{subfigure}{.5\textwidth}%
    \includegraphics[width=1.0\linewidth]{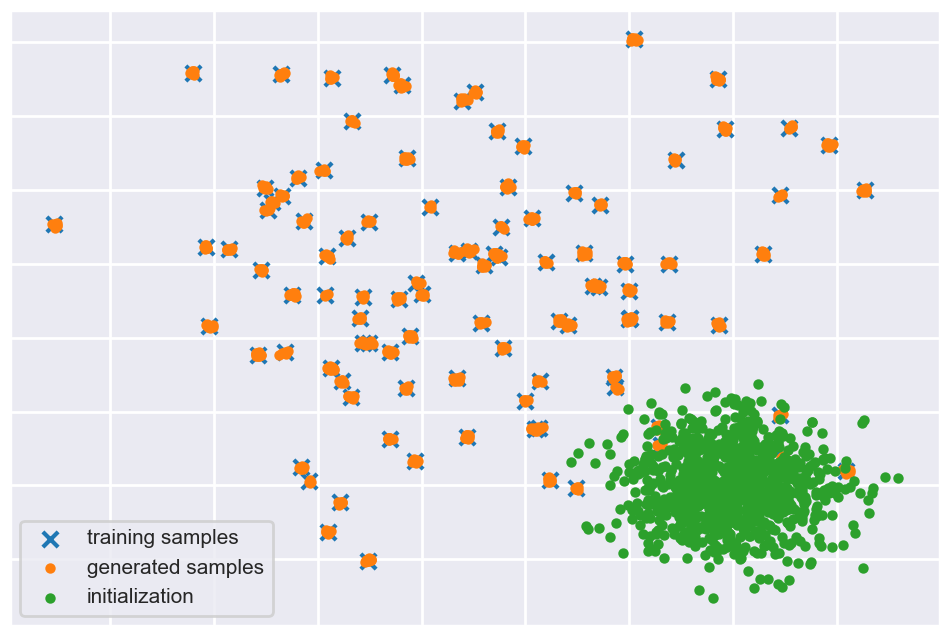}
    \end{subfigure}%
    \begin{subfigure}{.5\textwidth}%
    \includegraphics[width=1.0\linewidth]{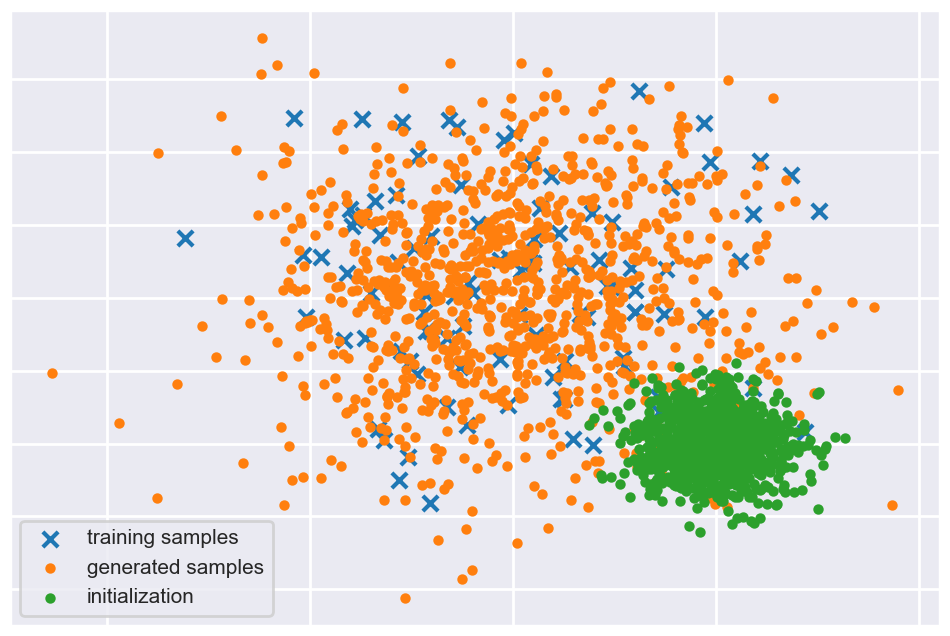}
    \end{subfigure}
\caption{\textbf{Left:} Samples generated by DDPM with \textit{empirical optimal score function} $s^N(t,x)$.  \textbf{Right:} Samples generated by DDPM with \textit{true score function} $u(t,x)$.
Both two algorithms are ran up to time $T = 5$, i.e. early stopping time $\delta = 0$.
The blue crosses are the training samples, the green dots are the initialization positions and the orange points are the generated samples.
}
\label{fig: generated samples 2d-gaussian without early stop}
\end{figure}

\begin{figure}[!htb]
    \centering
    \begin{subfigure}{.5\textwidth}%
    \includegraphics[width=1.0\linewidth]{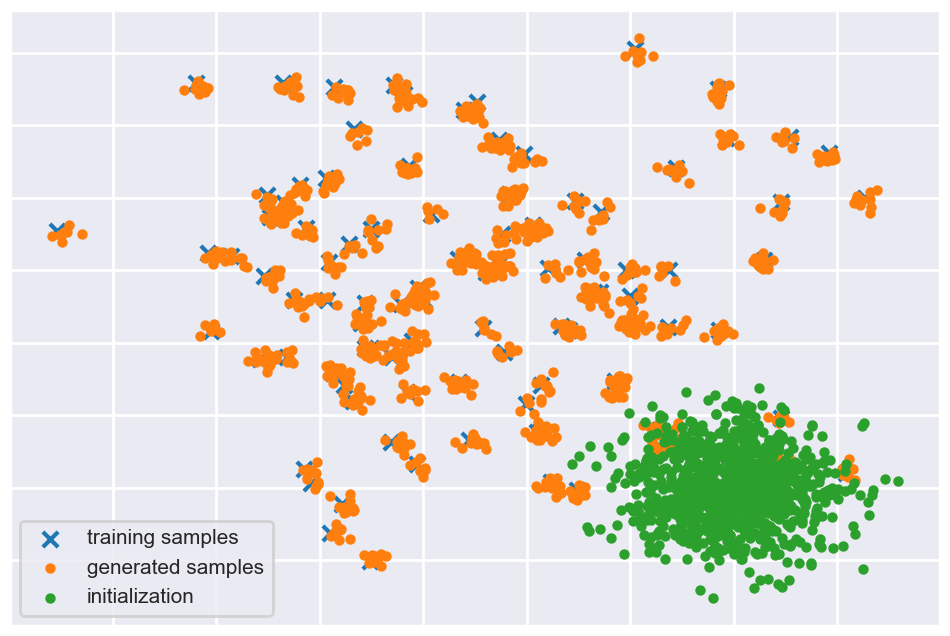}
    \end{subfigure}%
    \begin{subfigure}{.5\textwidth}%
    \includegraphics[width=1.0\linewidth]{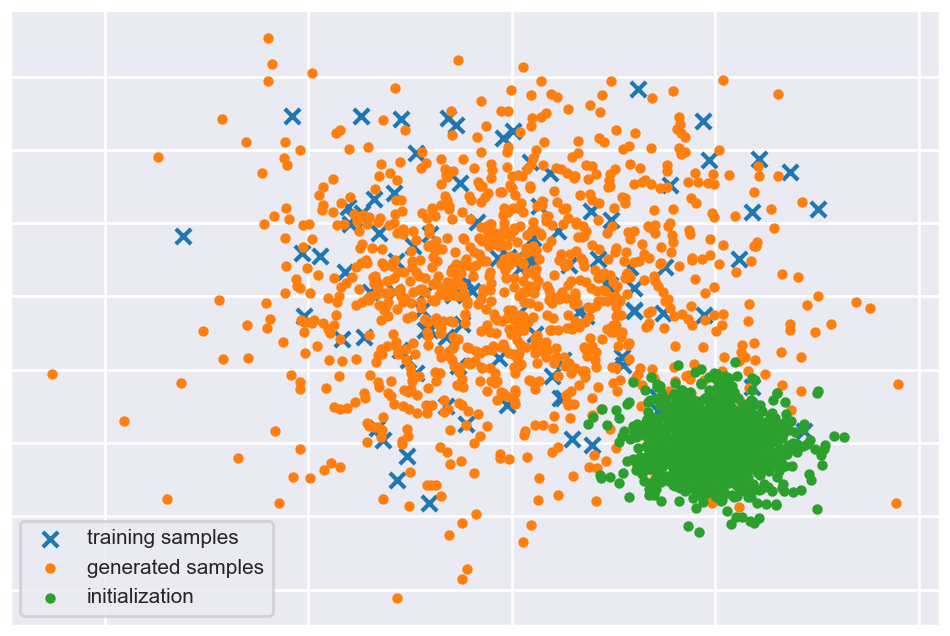}
    \end{subfigure}
\caption{\textbf{Left:} Samples generated by DDPM with \textit{empirical optimal score function} $s^N(t,x)$.  \textbf{Right:} Samples generated by DDPM with \textit{true score function} $u(t,x)$.
Both two algorithms are early stopped with $\delta = 0.01$.
The blue crosses are the training samples, the green dots are the initialization positions and the orange points are the generated samples.
}
\label{fig: generated samples 2d-gaussian with early stopping}
\end{figure}

\subsection{Real-world data distribution}
We consider $\ptarget$ as the underlying distribution generating the CIFAR10 dataset images \cite{krizhevsky2009learning},  conprising $N=50000$ training samples of dimension $d = 32 \times 32 \times 3$.
We denote $\{y_i\}_{i=1}^N$ as the $50000$ images in the CIFAR10 dataset, and we use them to construct the following two generative models.
\begin{itemize}
\item The first one is simple Gaussian Kernel Density Estimation (KDE), i.e. $\ptargetNgamma(x) = \frac{1}{N} \sum_{i=1}^N \mathcal{N}(x; y_i, \gamma^2 I_{d\times d})$, where $\gamma$ is the Gaussian kernel's bandwidth.
To generate a sample, we first uniformly sample a data $y_{(j)}$ from $\{y_i\}_{i=1}^N$, then apply Gaussian blurring with bandwidth $\gamma$ to $y_{(j)}$.
The bandwidth $\gamma$ is set to $0.1$ times the optimal bandwidth $N^{-\frac{1}{d+4}} \sigma$, as per Scott's rule \cite{terrell1992variable}, where  $\sigma$ is the training data's standard deviation.
The sampling results are shown in the second row of Figure \ref{fig: image generations from kde}.
Comparing with the training data (the first row in Figure \ref{fig: image generations from kde}), we can clearly see that the generated samples have strong dependence on existing ones.
    
\item The second one is DDPM equipped with the empirical optimal score function as defined in \eqref{eqn: DDPM with empirical optimal score function}.
We follow the implementations in \cite{gu2023memorization}.
To illustrate the details, we follow the notations used in \cite{gu2023memorization}.
Recall the backward SDE \eqref{eqn: DDPM with empirical optimal score function}
\begin{equation*}
d\widehat{X}_t^{\leftarrow} = (\widehat{X}_t^{\leftarrow} + 2 \scoreNy(T-t, \widehat{X}_t^{\leftarrow}) )dt + \sqrt{2}dB_t\,.
\end{equation*}
For sample generation, we discretize the time steps $0 = t_0 < t_1 < \cdots < t_K = T$ with $T > 0$ being the time interval length and $K > 0$ being the total number of steps, and apply the Euler-maruyama solver. 
The update rule is as the following:
\begin{equation}\label{eqn: discretized DDPM}
    X_{t_{n}} = X_{t_{n-1}} + \left(t_{n} - t_{n-1} \right) \left( X_{t_{n-1}} + 2 s^N_{\{y_i\}}(T-t_{n-1}, X_{t_{n-1}}) \right) + \sqrt{2 (t_{n} - t_{n-1})} Z,
\end{equation}
where $Z \sim \pi^d$.
We terminate this update rule \eqref{eqn: discretized DDPM} at $t_{\delta}$, where $\delta$ is the early stopping index\footnote{Here we abuse the notation $\delta$ and refer it as the early stopping index.
It is different from the $\delta$ we used in the main paper.}.
We set $T=80$, $K=18$, and vary $\delta$.
Figure \ref{fig: image generations from kde}'s third row shows the generated samples with $\delta=5$.
We can observe that the samples generated by DDPM equipped with the empirical optimal score function behave very similar to the samples generated by the Gaussian KDE (the second row in Figure \ref{fig: image generations from kde}).
This aligns with our theoretical findings provided in Theorem \ref{thm: DDPM resemble KDE}.
Additionally, Figure \ref{fig: image generations from ddpm} displays samples $\delta =3$ and $\delta =5$,
highlighting the strong memorization effect in DDPM with the empirical optimal score function, irrespctive of the early stopping time.

\begin{figure}[!htb]
\centering
\includegraphics[width=0.95\textwidth]{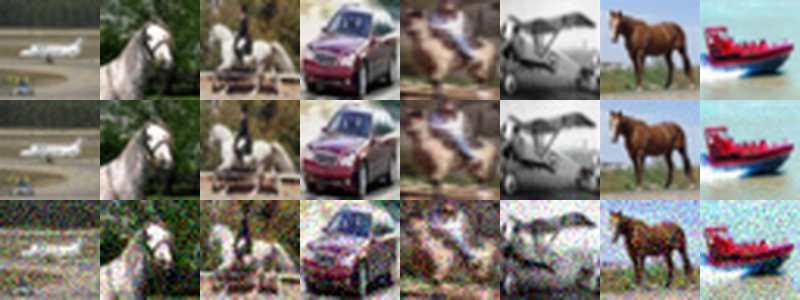}
\caption{Images generated by DDPM equipped with the empirical optimal score function based on CIFAR10 dataset. 
The first row is the original images from the CIFAR10 dataset. 
The second and third rows corresponding to the results of setting the early stopping index $\delta = 3$ and $5$ respectively.}
\label{fig: image generations from ddpm}
\end{figure}

\end{itemize}

\end{document}